\documentclass{article}
\usepackage{geometry}
\newgeometry{vmargin={1in}, hmargin={1in,1in}}

\usepackage[utf8]{inputenc} 
\usepackage[T1]{fontenc}    
\usepackage[colorlinks,
linkcolor=red,
anchorcolor=blue,
citecolor=blue,
urlcolor=blue
]{hyperref}      
\usepackage{url}            
\usepackage{booktabs}       
\usepackage{amsfonts,dsfont}       
\usepackage{nicefrac}       
\usepackage{natbib}
\usepackage{microtype}      
\usepackage{verbatim}

\usepackage{amssymb,amsmath,amsthm,bm}
\usepackage{booktabs}
\usepackage{threeparttable}
\usepackage{mathtools}
\usepackage{tablefootnote}

\usepackage{cuted}
\usepackage{lipsum, color}

\usepackage{algorithm}
\usepackage[noend]{algpseudocode}
\usepackage{enumitem}

\newtheorem{remark}{Remark}
\newtheorem{theorem}{Theorem}
\newtheorem{definition}{Definition}
\newtheorem{lemma}{Lemma}
\newtheorem{proposition}{Proposition}
\newtheorem{corollary}{Corollary}

\def\va{{\bm{a}}}

\def\vd{{\bm{d}}}

\def\vu{{\bm{u}}}
\def\vv{{\bm{v}}}

\def\vx{{\bm{x}}}
\def\vy{{\bm{y}}}

\def \Ib {{\mathbf{Ib}}}

\def \cS {{\mathcal{S}}}

\def \cL {{\mathcal{L}}}
\def \cG {{\mathcal{G}}}
\def \PP {{\mathbb{P}}}
\def \QQ {{\mathbb{Q}}}
\def \RR {{\mathbb{R}}}
\def \EE {{\mathbb{E}}}

\newcommand{\la}{\langle}
\newcommand{\ra}{\rangle}

\def \bX {{\bm X}}
\def \bx {{\bm x}}
\def \bd {{\bm d}}
\def \xb {{\bm x}}
\def \yb {{\bm y}}
\def \bc {{\bm c}}
\def \bb {{\bm b}}
\def \bB {{\bm B}}

\def \Ib {{\mathbf{I}}}
\def \Ab {{\mathbf{A}}}
\def \Bb {{\mathbf{B}}}

\def \Qb {{\mathbf{Q}}}

\def \Gb {{\mathbf{G}}}
\def \gb {{\mathbf{g}}}
\def \ind {\mathds{1}}
\def \bepsilon {{\boldsymbol{\epsilon}}}
\def \cN {{\mathcal {N}}}

\def \cW {{\mathcal{W}}}
\def \bY{{\bm Y}}
\def \bV{{\bm V}}
\def \bW{{\bm W }}

\DeclareMathOperator*{\argmax}{arg\,max}

\usepackage{soul}
\newtheorem{assumption}{Assumption}

\def \dd {\text{d}}

\newcommand{\commbw}[1]{{\color{red}(BW: #1)}} 

\title{
\huge Laplacian Smoothing Stochastic Gradient Markov Chain Monte Carlo
}

\author{
  Bao Wang$^*$ \\
  Department of Mathematics\\
  University of California, Los Angeles\\
  \texttt{wangbaonj@gmail.com}\\
  \and
  Difan Zou\footnote{Equal Contribution}\\
  Department of Computer Science\\
  University of California, Los Angeles\\
  \texttt{knowzou@cs.ucla.edu} \\
  \and
  Quanquan Gu$^\dag$ \\
  Department of Computer Science\\
  University of California, Los Angeles\\
  \texttt{qgu@cs.ucla.edu}\\
  \and
  Stanley J. Osher\footnote{Co-Corresponding Authors} \\
  Department of Mathematics\\
  University of California, Los Angeles\\
   \texttt{sjo@math.ucla.edu} \\
}

\begin{document}

\maketitle

\begin{abstract}
As an important Markov Chain Monte Carlo (MCMC) method, stochastic gradient Langevin dynamics (SGLD) algorithm has achieved great success in Bayesian learning and posterior sampling. However, SGLD typically suffers from slow convergence rate due to its large variance caused by the stochastic gradient. In order to alleviate these drawbacks, we leverage the recently developed Laplacian Smoothing (LS) technique and propose a Laplacian smoothing stochastic gradient Langevin dynamics (LS-SGLD) algorithm. We prove that for sampling from both log-concave and non-log-concave densities, LS-SGLD achieves strictly smaller discretization error in $2$-Wasserstein distance, although its mixing rate can be slightly slower. Experiments on both synthetic and real datasets verify our theoretical results, and demonstrate the superior performance of LS-SGLD on different machine learning tasks including posterior sampling, Bayesian logistic regression and training Bayesian convolutional neural networks. The code is available at \url{https://github.com/BaoWangMath/LS-MCMC}.
\end{abstract}


\section{Introduction}
Given a dataset $\mathcal{D}=\{\bd_i\}_{i=1}^n$, the posterior of a machine learning (ML) model's parameters $\bx \in \RR^d$ with prior $p(\bx)$ and likelihood $\Pi_{i=1}^n p(\bd_i|\bx)$ is computed as $p(\bx|\mathcal{D})\propto p(\bx)\Pi_{i=1}^n p(\bd_i|\bx)$. Optimization algorithms are used to find the maximum a posterior (MAP) estimate, $\bx_{\rm MAP}=\argmax_\bx \log p(\bx|\mathcal{D})$. Sampling algorithms such as 
Langevin Dynamics (LD) are used to sample the posterior 
or the log posterior. 
In this paper, we consider applying LD-based MCMC algorithms to sample $e^{-f(\bx)}$, where
\begin{align}\label{Eq:Posterior}
f(\bx) := \frac{1}{n}\sum_{i=1}^n f_i(\bx) = -\frac{1}{n}\sum_{i=1}^n \log p(\bd_i|\bx).
\end{align}
Here, we normalized the log-likelihood by a factor $n$ for the ease of presentation in the remaining part of this paper.

The first-order LD reads:
\begin{align}\label{eq:sde_langevin}
d\bX_t = -\nabla f(\bX_t) \dd t + \sqrt{2\beta^{-1}}\cdot \dd \bB_t,
\end{align}
where $\bX_t\in \RR^d$ denotes the point at time $t$, $\beta$ denotes the inverse temperature and $\bB_t\in \RR^d$ is the standard Brownian term. Under certain assumptions on the negative log posterior (i.e., $f(\xb)$), the LD \eqref{eq:sde_langevin} converges to an unique invariant distribution $\pi\propto e^{-\beta f(\xb)}$ \citep{chiang1987diffusion}. Therefore, one can apply numerical integrator to approximate \eqref{eq:sde_langevin} in order to obtain samples that follow the posterior distribution. One simple integrator is to apply the Euler-Maruyama discretization \citep{Kloeden:1992} to \eqref{eq:sde_langevin}, which gives:
\begin{align}\label{eq:lmc}
\xb_{k+1} = \xb_k - \eta \nabla f(\xb_k) + \sqrt{2\beta^{-1} \eta}\cdot \bepsilon_k,
\end{align}
and it is known as the Langevin Monte Carlo (LMC), (a.k.a., unadjusted Langevin algorithm \citep{parisi1981correlation}). When the target density, i.e., posterior distribution, is strongly log-concave and log-smooth, \citet{dalalyan2017theoretical,durmus2017nonasymptotic} proved that LMC is able to converge to the target density up to an arbitrarily small sampling error in both total variation and $2$-Wasserstein distances. Furthermore, the convergence guarantee of LMC for sampling from non-log-concave distributions has also been established in \citet{raginsky2017non,xu2018global}.
\medskip

\noindent Note that the posterior distribution is defined on the whole dataset $\mathcal{D}$, which is typically extremely large in modern ML tasks. Therefore, computing the full gradient $\nabla f(\xb)$ is inefficient and may dramatically slow down the convergence of sampling algorithms. One solution is to replace the full gradient in \eqref{eq:lmc} with a subsampled one, which gives rise to Stochastic Gradient Langevin Dynamics (SGLD) \citep{Welling:2011SGLD}. From the theoretical perspective, the convergence guarantee of SGLD has been proved for both strongly log-concave distributions \citep{dalalyan2017user} and non-log-concave distributions \citep{raginsky2017non,xu2018global} in $2$-
Wasserstein distance. \citet{mou2017generalization} further studied the generalization performance of SGLD for nonconvex optimization. Although SGLD can drastically reduce the computational cost, it is also observed to have a slow convergence rate due to the large variance caused by the stochastic gradient \citep{teh2016consistency,vollmer2016exploration}. In order to reduce the variance of stochastic gradient as well as to improve the convergence rate, \citet{dubey2016variance} incorporated variance reduction techniques into SGLD, which gives rise to a family of variance-reduced LD-based algorithms such as SVRG-LD and SAGA-LD. \citet{chatterji2018theory,zou2018subsampled,zou2019sampling} further proved that SVRG-LD and SAGA-LD are able to converge to the target density with fewer stochastic gradient evaluations than SGLD and LMC in certain regimes. However, both SVRG-LD and SAGA-LD require a large amount of extra computation and memory costs and can only be shown to achieve faster convergence on small to moderate datasets.  Therefore, it is natural to ask \emph{if we can reduce the variance of stochastic gradients while maintaining similar computation and memory costs of SGLD?}
\medskip

\noindent Recently, \cite{LS-GD:2018} integrated Laplacian smoothing and related high-order smoothing techniques into Stochastic Gradient Descent (SGD) to reduce the variance of stochastic gradient on-the-fly. 
Laplacian Smoothing SGD (LSSGD) allows us to take a significantly larger step size than vanilla SGD and reduces the optimality gap in convex optimization when constant step size is used. Empirically, LSSGD  preconditions the gradient when the objective function has a large condition number and can avoid local minima. Because of this, LSSGD is applicable to train a large number of deep learning models with good generalizability. Laplacian smoothing also demonstrates some ability to avoid saddle point in gradient descent \citep{Kreusser:2019}. Most recently, \cite{Wang:2019:DPLSSGD} leveraged Laplacian smoothing to improve the utility of machine learning models trained with privacy guarantee. 
\medskip

\noindent In this paper, we integrate Laplacian smoothing with SGLD, and we call the resulting algorithm Laplacian Smoothing SGLD (LS-SGLD). The extra computation of LS-SGLD compared with SGLD is that we need to compute the products of the inverse of two circulant matrices with vectors. We leverage the Fast Fourier Transform to develop fast algorithms to compute these matrix-vector products efficiently, and the resulting algorithms can 
compute the matrix-vector products with a negligible overhead in both time and memory.

\subsection{Our Contributions}
We summarize the main contributions of this work as follows:
\begin{itemize}
\item We propose a simple modification on the SGLD, which applies the Laplacian smoothing matrix and its squared root to the stochastic gradient and Gaussian noise vectors, respectively. The continuous and full-gradient counter-part of the modified LS-SGLD has the same stationary distribution as the LD.

\item We proposed FFT-based fast algorithms to compute the product of the inverse of circulant matrices with any given vector. By leveraging the structure of eigenvalues and eigenvectors of the circulant matrices, we can compute these products very efficiently with a negligible overhead in both time and memory.

\item We prove the convergence rate of LS-SGLD for sampling from both log-concave and non-log-concave densities in $2$-Wasserstein distance. Specifically, we decompose the sampling error into the  discretization error and the ergodicity rate. Moreover, we show that there exists a trade-off between the discretization error and ergodicity rate of LS-SGLD, as adding Laplacian smoothing can reduce the discretization error but slow down the mixing time.

\item We conduct extensive experiments to evaluate the performance of LS-SGLD.  First, we show that compared with SGLD, LS-SGLD can achieve a significantly smaller discretization error but similar ergodicity rate, which implies that the overall sampling error of LS-SGLD can be much smaller. Second, we conduct experiments on both synthetic and real data for posterior sampling, Bayesian logistic regression and training Bayesian convolutional networks, all of which demonstrate the superior performance of LS-SGLD.

\end{itemize}

\subsection{Additional Related Work}\label{subsection:RelatedWorks}
In addition to the first-order Langevin based algorithms we discussed in the introduction, there also emerges a vast body of work focusing on higher-order Langevin based algorithms. One of the well-know high-order MCMC method is Hamiltonian Monte Carlo (HMC) \citep{neal2011mcmc}, which  incorporates an Hamiltonian momentum term into the first-order MCMC method in order to improve the mixing time. 
Similar to SGLD, a stochastic version of HMC (namely SGHMC) has been further established in \citet{Chen:2014SGHMC}, which was shown to be able to achieve a faster convergence rate than SGLD in experiments. \citet{ma2015complete} investigated a family of SGHMC methods and proposed a new state-adaptive sampler
on the Riemannian manifold. \citet{chen2015convergence} provided theoretical convergence guarantees of SGHMC in terms of mean square error (MSE) and proposed a 2nd-order symmetric splitting integrator to further improve the discretization error. 
When the target density is strongly log-concave and log-smooth, \citet{cheng2018underdamped} proposed underdamped MCMC (U-MCMC) and stochastic underdamped MCMC (SG-U-MCMC), and obtained convergence rates in $2$-Wasserstein distance. The convergence rates of these two algorithms have been further established for sampling from non-log-concave densities \citep{cheng2018sharp}.  However, due to the large variance of stochastic gradients and lacking of the Metropolis Hasting (MH) correction step, SGHMC has also been observed to have highly biased sampling trajectory  \citep{betancourt2015fundamental,dang2019hamiltonian}. One way to address this issue is to make use of a variance-reduction technique to alleviate the variance of stochastic gradients in SGHMC, which gave rise to stochastic variance-reduced HMC methods \citep{zou2018stochastic,li2018stochastic,zou2019stochastic}.





\subsection{Organization}
We organize this paper as follows: We present LS-SGLD and derive FFT-based fast algorithms for LS-SGLD in Section~\ref{section:Algorithm}. In Section~\ref{section:Theorems}, we give theoretical guarantees for the performance of LS-SGLD in both log-concave and non-log-concave settings. In Section~\ref{section:Experiments}, we numerically verify the performance of LS-SGLD on sampling different distributions, training Bayesian logistic regression, and convolutional neural nets. We conclude this work in Section~\ref{section:Conclusion}.

\subsection{Notations}
Throughout this paper we use bold upper-case letters $\Ab$, $\Bb$ to denote matrices, bold lower-case letters $\vx$, $\vy$ to denote vectors, and lower cases letters $x$, $y$ and $\alpha$, $\beta$ to denote scalars. For continuous-time random vectors, we denote them with the tilt bold upper-case letters $\bX$, $\bY$ with sub/super-scripts. For vector $\xb = (x_1,\dots,x_d)^\top$, we use $\|\xb\|_2 = \sqrt{x_1^2+\dots+x_d^2}$ to represent its $\ell_2$-norm and use $\|\xb\|_{\Ab}=\sqrt{\xb^\top\Ab\xb}$ to represent its $\Ab$-norm, where $\Ab$ is an semi-positive definite matrix. We use $\PP(\bx)$ to denote the distribution of $\bx$, and $\cW_2(\cdot, \cdot)$ and $D_{KL}(\cdot||\cdot)$ denote the $2$-Wasserstein distance and Kullback–Leibler (KL) divergence between two distributions, respectively. For a function $f: \RR^d\rightarrow \RR$, we use $\nabla f(\cdot)$ and $\nabla^2 f(\cdot)$ to denote its gradient and Hessian.

\section{Algorithms}\label{section:Algorithm}
\subsection{Laplacian Smoothing (Stochastic) Gradient Descent}
For $\sigma \geq 0$, let $\Ab_\sigma := \Ib - \sigma\mathbf{L}$ where $\Ib \in \RR^{d\times d}$ and $\mathbf{L} \in \RR^{d\times d}$ are the identity and the discrete 
one-dimensional Laplacian matrix, respectively. Therefore,
\begin{equation}\label{eq:tri-diag}
\Ab_\sigma := 
\begin{bmatrix}
1+2\sigma   & -\sigma &  0&\dots &0& -\sigma \\
-\sigma     & 1+2\sigma & -\sigma & \dots &0&0 \\
0 & -\sigma  & 1+2\sigma & \dots & 0 & 0 \\
\dots     & \dots & \dots &\dots & \dots & \dots\\
-\sigma     &0& 0 & \dots &-\sigma & 1+2\sigma
\end{bmatrix}_{d\times d}
\end{equation}
To find $\bx_{\rm MAP}$ of \eqref{Eq:Posterior}, LSSGD \citep{LS-GD:2018} takes the following iteration
\begin{equation}
\label{LSSGD}
\bx^{k+1} = \bx^k -\eta_k \Ab_\sigma^{-1}\nabla f_{i_k}(\bx^k),
\end{equation}
where $\eta_k > 0$ is the learning rate, 
$i_k$ is a random sample from $[n] := \{1, 2, \cdots, n\}$. 
When $\sigma=0$, LSSGD reduces to SGD. Since $\Ab_\sigma$ is a circulant matrix, for any vector $\vv$, $\Ab_\sigma^{-1}\vv := \vu$ can be computed via the FFT in the following way
$$
\Ab_\sigma^{-1} \vv = \vu \Longrightarrow \vv = \Ab_\sigma\vu = \vu - \sigma\vd *\vu,
$$
where $*$ is the convolution operator, and $\vd = [-2, 1, 0, \cdots, 0, 1]^T$. 
By the convolution theorem, we have
$$
{\rm fft}(\vv) = {\rm fft}(\vu) - \sigma{\rm fft}(\vd){\rm fft}(\vu).
$$
Finally, we arrive at the following FFT-based algorithm for computing $\Ab_\sigma^{-1}\vv$
$$
\Ab_\sigma^{-1}\vv =  {\rm ifft}\left(\frac{{\rm fft}(\vv)}{\mathbf{1} -\sigma \cdot {\rm fft}(\vd)}\right),
$$
where $\mathbf{1}$ is an all-one vector with the same dimension as $\vv$, and the division of two vectors 
is defined in the coordinate-wise way. fft and ifft denote FFT and inverse FFT operators, respectively.
\medskip

\noindent The Laplacian matrix $\Ab_\sigma^{-1}$ can reduce the variance of stochastic gradient and guarantee at least the same convergence rate as SGD. \cite{LS-GD:2018} showed that for a $L$-gradient Lipschitz function $f(\bx)$, i.e., $\|\nabla f(\bx)\|_2 \leq L$, the largest step size for LSSGD is $(1+4\sigma)^{1/4}/L$ (with high probability) which is larger than GD's by a factor $(1+4\sigma)^{1/4}$.

\subsection{Laplacian Smoothing 
Langevin Dynamics}

We integrate Laplacian smoothing with LD and obtain the following Lapacian Smoothing LD (LS-LD)
\begin{align}\label{eq:langevin_cont}
d\bX_t = -\Ab_\sigma^{-1}\nabla f(\bX_t) + \sqrt{2\beta^{-1}}\Ab_\sigma^{-1/2}\dd \bB_t.
\end{align}

Note that we pre-multiply the Brownian motion term by $\Ab_\sigma^{-1/2}$ instead of $\Ab_\sigma^{-1}$ to guarantee that the stationary distribution of the LS-LD remains to be $\exp{\left(-\beta f(\xb)\right)}$. We can easily verify that LD and LS-LD have the same stationary distribution by looking at the associated Fokker-Planck equation. We formally state this property in the following proposition.



\begin{proposition}\label{Prop:StationaryDistribution}
The stationary distribution, $\pi$, of the LS-LD, \eqref{eq:langevin_cont}, satisfies $\pi\propto e^{-\beta f(\xb)}$.
\end{proposition}

\noindent If we apply the Euler-Maruyama scheme to discretize \eqref{eq:langevin_cont}, we end up with the following discrete algorithm, namely Laplacian smoothing gradient Langevin dynamics (LS-GLD)
\begin{equation}\label{Eq:LS-LD}
\bx_{k+1} = \bx_k - \eta \Ab_\sigma^{-1} \nabla f(\bx_k) + \sqrt{2\beta^{-1}\eta}\Ab_\sigma^{-1/2}\boldsymbol{\epsilon}_k,
\end{equation}
where $\boldsymbol{\epsilon}_k\sim N(\mathbf{0}, \Ib_{d\times d})$. In practice, we use the  mini-batch gradient 
$\gb_k= \sum_{i\in\mathcal I_k} \nabla f_{i}(\xb_k)/|\mathcal I_k|$ 
with $\mathcal I_k \subset [n]$ to replace the gradient in \eqref{Eq:LS-LD}, and we arrive at the following LS-SGLD
\begin{equation}\label{Eq:LS-SGLD}
\bx_{k+1} = \bx_k - \eta \Ab_\sigma^{-1} \gb_k + \sqrt{2\beta^{-1}\eta}\Ab_\sigma^{-1/2}\boldsymbol{\epsilon}_k.
\end{equation}

We summarize LS-SGLD in Algorithm~\ref{LS-SGLD-Pseudocode}. It is worth noting that computing the inverse of both $\Ab_\sigma$ and $\Ab_\sigma^{1/2}$ can be expensive. Moreover, multiplying the vectors by the inverse of these two matrices is also expensive. So in the remaining part of this section, we will present FFT-based fast algorithms for implementing \eqref{Eq:LS-LD}.
\begin{algorithm}[!ht]
\caption{LS-SGLD}\label{LS-SGLD-Pseudocode}
\begin{algorithmic}
\State \textbf{Input: } Training data, learning rate $\eta$, minibatch size $B$, inverse temperature $\beta$, Laplacian smoothing constant $\sigma$.
\State \textbf{Initialization:} Set $\xb_0 = 0$.
\For {$k=0, 1, \cdots, K-1$}
\State Uniformly sample $\mathcal{I}_k\subset [n]$ with $|\mathcal{I}_k| = B$.
\State Compute the mini-batch stochastic gradient 
$\sum_{i\in\mathcal I_k} \nabla f_{i}(\xb_k)/B $.
\State $\bx_{k+1} = \bx_k - \eta\Ab_\sigma^{-1}\gb_k+\sqrt{2\beta^{-1}\eta}\Ab^{-1/2}_\sigma\boldsymbol{\epsilon}_k$, where $\boldsymbol{\epsilon}_k\sim N(\mathbf{0}, \mathbf{I}_{d\times d})$.
\EndFor
\State \textbf{Output: }  $\bx_0,\dots,\bx_K$.
\end{algorithmic}
\end{algorithm}

\subsection{FFT-based Implementation of LS-SGLD}
\subsubsection{Circulant Matrix and Convolutional Operation}
In this subsection, we list a few results on the circulant matrix which will be the basic recipes for designing FFT-based algorithm to solving \eqref{Eq:LS-LD}.
\begin{lemma}[\cite{Golub:1996}]\label{Lemma:Eigens}
The normalized eigenvectors of the following $d\times d$ circulant matrix, 
\begin{equation}\label{Eq:Circulant}
\mathbf{C} = 
\begin{bmatrix}
c_0 & c_{d-1} &\dots & c_2 & c_1 \\
c_1 & c_0     & c_{d-1}  &\dots &c_2 \\
\dots & \dots & \dots  &\dots &\dots \\
c_{d-2} & \dots     & \dots  &\dots &c_{d-1} \\
c_{d-1} & c_{d-2}   & \dots  &c_1 &c_0 
\end{bmatrix},
\end{equation}
are given by
$$
\vv_j = \frac{1}{\sqrt{d}}\left(1, w_j, w_j^2, \cdots, w_j^{n-1}\right), \ \ j = 0, 1, \cdots, d-1,
$$
where $w_j=\exp\left(i\frac{2\pi j}{d}\right)$ are the $j$-th roots of unity and $i$ is the imaginary unit. The corresponding eigenvalues are then given by
$$
\lambda_j = c_0 + c_{d-1}w_j + c_{d-2}w_j^2 + \cdots + c_1w_j^{d-1}, \ \ \ j = 0, 1, \cdots, d-1.
$$
\end{lemma}

\begin{lemma}[\cite{Golub:1996}]\label{Lemma:Inverse-Circulant}
The inverse of a circulant matrix is circulant.
\end{lemma}

\begin{lemma}[\cite{Golub:1996}]\label{Lemma:SquareRoot-Circulant}
The square root of a circulant matrix is circulant.
\end{lemma}

\begin{lemma}\label{Lemma:CirculantInverse}
For any circulant matrix $\mathbf{C}$ of the form in \eqref{Eq:Circulant}, and for any given vector $\vv$. Let $\vu = \mathbf{C}^{-1}\vv$, then $\vu$ can be computed by the fast Fourier transform with sublinear scaling in the following way
\begin{equation}\label{Eq:fft-matrix-vec-product}
\vu = {\rm ifft}\left(\frac{{\rm fft}(\vv) }{{\rm fft}(\bc)}\right),
\end{equation}
where $\bc$ is the first row of the matrix $\mathbf{C}$,
and the division in \eqref{Eq:fft-matrix-vec-product} is defined coordinate-wise.
\end{lemma}
\begin{proof}
Since $\mathbf{C}$ is a circulant matrix we have $\vv = \mathbf{C}\vu = \bc * \vu$, therefore ${\rm fft}(\vv) = {\rm fft}(\bc)\cdot {\rm fft}(\vu)$.
\end{proof}

\subsubsection{Fast Algorithm for Computing the Square Root of Laplacian Smoothing}
We will derive an FFT-based algorithm for computing
$\Ab_\sigma^{-1/2}\boldsymbol{\epsilon}_k$ in this subsection.
According to Lemmas~\ref{Lemma:Inverse-Circulant} and \ref{Lemma:SquareRoot-Circulant}, $\Ab_\sigma^{-1/2}$ is circulant. Note $\Ab_\sigma^{-1}$ is positive definite, we denote its eigen-decomposition as
$$
\Ab_\sigma^{-1} = \Qb\Lambda \Qb^{-1},
$$
where $\Qb=[\vv_1, \vv_2, \cdots, \vv_d]^T$ with $\vv_i$ being the eigenvector associated with the eigenvalue $\lambda_i > 0$, and $\Lambda = {\rm diag}(\lambda_1, \lambda_2, \cdots, \lambda_d)$.
Therefore, we have
\begin{equation}\label{Eq:SqrtMatrix}
\Ab_\sigma^{-1/2} = \Qb\sqrt{\Lambda}\Qb^{-1},
\end{equation}
where $\sqrt{\Lambda} = {\rm diag}(\sqrt{\lambda_1}, \sqrt{\lambda_2}, \cdots, \sqrt{\lambda_d})$.
\medskip

\noindent Furthermore, note that $\Ab_\sigma$ is symmetric, therefore $\Qb^{-1} = \Qb^T$. It follows that we can compute $\Ab^{-1/2}_\sigma$ without inverting the matrix $\Qb$. By the fact that $\Ab_\sigma^{-1/2}$ is circulant, we have $\Ab_\sigma^{-1/2}\boldsymbol{\epsilon}_k = {\rm ifft}({\rm fft}(\bb)\cdot {\rm fft}(\boldsymbol{\epsilon}_k))$, where $\bb$ is the first row of $\Ab_\sigma^{-1/2}$.
\medskip

\begin{remark}
In computing \eqref{Eq:SqrtMatrix}, there is no need to store the matrix $\Qb$,  according to Lemma~\ref{Lemma:Eigens}, each row of $\Qb$ and $\sqrt{\Lambda}$ can be written down explicitly which enables us to compute $\Ab_\sigma^{-1/2}$ quickly with negligible memory overhead and scalable to very high dimensional problems.
\end{remark}

\section{Main Results}\label{section:Theorems}
We first make the following three assumptions regarding the function $f(\xb)$.
\begin{assumption}[Dissipativeness]\label{assump:dissipative}
For any $\xb\in\RR^d$, there exist constants $m$ and $b$ such that
\begin{align*}
\la\nabla f(\xb),\xb\ra\ge m\|\xb\|_2^2 - b.
\end{align*}
This assumption has been widely made to study the convergence of Langevin based sampling algorithms \citep{mattingly2002ergodicity,raginsky2017non,xu2018global,zou2019sampling}, which is essential to guarantee the convergence of the continuous-time Langenvin dynamics \eqref{eq:sde_langevin}.


\end{assumption}

\begin{assumption}[smoothness]\label{assump:smooth}
For any $\xb,\yb\in\RR^d$, there exists a positive constant $M$ such that for all $i=1,\dots,n$, it holds that
\begin{align*}
\|\nabla f_i(\xb) - \nabla f_i(\yb)\|_2\le M\|\xb - \yb\|_2.
\end{align*}
Unlike Assumption \ref{assump:dissipative}, Assumption \ref{assump:smooth} is made for all component function $f_i(\xb)$.

\end{assumption}
\begin{assumption}[Bounded Variance]\label{assump:bound_var} For any $\xb \in \RR^d$, there exists a constant $\omega$ such that the variance of the stochastic gradient is bounded as follows,
\begin{align*}
\EE[\|\nabla f_i(\xb) - \nabla f(\xb)\|_2^2]\le d \omega^2.
\end{align*}

\end{assumption}

\begin{definition}[Logarithmic Sobolev inequality]
Let $\mu$ be a probability measure, then we say $\mu$ satisfies logarithmic Sobolev inequality with constant $\lambda$ if
for any smooth function $g$, the following holds:
\begin{align*}
\int g^2\log g^2 \dd\mu - \int g^2 \dd \mu \log \int g^2\dd \mu \le \lambda \int \|\nabla g\|_2^2 \dd \mu.
\end{align*}
\end{definition}
\medskip
\noindent Then the following proposition states that if the function $f(\cdot)$ satisfies Assumptions \ref{assump:dissipative} and \ref{assump:smooth}, the target density $\pi\propto e^{-f(\xb)}$ satisfies Logarithmic Sobolev inequality.
\begin{proposition}[\citet{raginsky2017non}]\label{prop:sobo}
Under Assumptions \ref{assump:dissipative} and \ref{assump:smooth}, the target density $\pi\propto e^{-f(\xb)}$ satisfies Logarithmic Sobolev inequality with some constant $\lambda>0$.
\end{proposition}
\noindent It has been shown in \citet{durmus2017nonasymptotic,raginsky2017non} that if the function $f(\xb)$ is smooth and strongly convex (which is stronger than Assumption \ref{assump:dissipative}), the logarithmic Sobolev constant $\lambda$ is an universal constant. However, if the function $f(\xb)$ is nonconvex, in the worst case the logarithm Sobolev constant $\lambda$ can have exponential dependency on the problem dimension $d$ and inverse temperature $\beta$ \citep{bovier2004metastability,raginsky2017non}.

\subsection{Convergence Analysis of Sampling from Log-concave Densities}

In this subsection, we 
assume that the target density is log-concave, which is equivalent to the following assumption on the function $f(\xb)$. 
\begin{assumption}[Convexity]\label{assump:convex} 
For any $\xb, \yb \in \RR^d$, it holds that
\begin{align*}
f(\xb) - f(\yb)\ge \la\nabla f(\yb), \xb - \yb\ra.
\end{align*}

\end{assumption}

Then we are ready to establish the convergence rate of LS-SGLD for sampling from log-concave densities, which is stated in the following theorem.
\begin{theorem}\label{thm:main_theory_convex}
Under Assumptions \ref{assump:dissipative}, \ref{assump:smooth}, \ref{assump:bound_var} and \ref{assump:convex}, if set the step size $\eta \le C m\beta^{-1}/M^2$ for some sufficiently small constant $C$, there exist constants $c_0\in[\|\Ab_\sigma\|_2^{-1}, 1]$, $\gamma_1\in[\|\Ab_\sigma\|_2^{-2}, 1]$ and $\gamma_2 = d^{-1}\sum_{i=1}^d(1+2\sigma - 2\sigma \cos(2\pi i/d))^{-1}$ such that the output of LS-SGLD satisfies,
\begin{align}\label{eq:bound_convex}
\cW_2(\PP(\xb_{K}), \pi) \le \bigg(\frac{2\gamma_1 K\eta^2\beta d\omega^2}{B}\bigg)^{1/2} + \big[ 8\gamma_2 K\eta^2\cdot(K+1) \beta d\eta \big]^{1/2} + \big[2\lambda\big(\beta f(0) + \log(\Lambda)\big)\big]^{1/2}\cdot e^{-c_0K\eta/(2\beta \lambda)},
\end{align}
where $\Lambda = \int_{\RR^d} e^{-\beta f(\xb)} \dd\xb$ and $\lambda$ denotes the logarithmic Sobolev constant of the target distribution $\pi\propto e^{-\beta f(\xb)}$.
\end{theorem}


\begin{table}[!t]
\centering
\fontsize{10}{10}\selectfont
\begin{threeparttable}
\caption{The values of $\gamma_2$ corresponding to some $\sigma$ and $d$.}\label{Beta-Table}
\begin{tabular}{cccccc}
\toprule[1.0pt]
$\sigma$   &\ \ \ \ \ \ \ \  1 \ \ \ \ \ \ \ \ &\ \ \ \ \ \ \ \  2\ \ \ \ \ \ \ \  &\ \ \ \ \ \ \ \  3\ \ \ \ \ \ \ \  &\ \ \ \ \ \ \ \  4\ \ \ \ \ \ \ \  &\ \ \ \ \ \ \ \  5\ \ \ \ \ \ \ \   \cr
\midrule[0.8pt]
$d=1000$    & 0.268 & 0.185 & 0.149 & 0.128 & 0.114 \cr
$d=10000$   & 0.268 & 0.185 & 0.149 & 0.128 & 0.114 \cr
$d=100000$  & 0.268 & 0.185 & 0.149 & 0.128 & 0.114 \cr
\bottomrule[1.0pt]
\end{tabular}
\end{threeparttable}
\end{table}


\begin{remark}
We emphasize that the the three terms on the R.H.S. of \eqref{eq:bound_convex} have their respective meanings. In particular, the first and second terms represent the discretization errors introduced by stochastic gradient estimator and numerical integrator of \eqref{eq:langevin_cont}, respectively.  The third term represents the ergodicity of the continuous-time Markov process \eqref{eq:langevin_cont}, which characterizes the mixing time of  LS-LD \eqref{eq:langevin_cont}. Moreover, we remark here that the convergence rate of LS-GLD (LS-SGLD with full gradient) can be directly implied from
Theorem \ref{thm:main_theory_convex} by removing the first term on the R.H.S. of \eqref{eq:bound_convex}. 
\end{remark}
Based on Theorem \ref{thm:main_theory_convex}, we can also derive the convergence rate of SGLD in the same setting by setting $\Ab_\sigma = \Ib$ (i.e., $\sigma = 0$), which implies that the constants $\gamma_1, \gamma_2$ and $c_0$ in Theorem \ref{thm:main_theory_convex} are all $1$'s. We formally state the convergence result of SGLD in the following corollary.

\begin{corollary}
Under the same assumptions in Theorem \ref{thm:main_theory_convex},  the output of standard SGLD, denoted by $\yb_K$, satisfies
\begin{align}\label{eq:bound_convex_sgld}
\cW_2(\PP(\yb_{K}), \pi) \le \bigg(\frac{2 K\eta^2 d\omega^2}{B}\bigg)^{1/2} + \big[ 8 K\eta^2\cdot(K+1) \beta^{-1} d\eta \big]^{1/2} + \big[2\lambda\big(\beta f(0) + \log(\Lambda)\big)\big]^{1/2}\cdot e^{-c_0K\eta/(2\beta \lambda)},
\end{align}
\end{corollary}
\begin{remark}
We can now compare the convergence rates of LS-SGLD and SGLD. In terms of the discretization error, it is clear that LS-SGLD is strictly better since the constants $\gamma_1$ and $\gamma_2$ are strictly less than $1$ (some values of $\gamma_2$ corresponding to different choices of $\sigma$ and $d$ can be found in Table \ref{Beta-Table}). In terms of the ergodicity of the continuous-time Markov process (the third terms in \eqref{eq:bound_convex} and \eqref{eq:bound_convex_sgld}), LS-SGLD is worse than SGLD due to the fact that $c_0\le1$. Therefore, there exists a trade-off between the discretization error and the ergodicity rate of LS-SGLD. In our experiments we will conduct numerical evaluations of these error terms and demonstrate that LS-LD and LD achieve similar ergodicity performance (i.e., mixing time), but LS-SGLD can achieve a significantly smaller discretization error.
\end{remark}

\subsection{Convergence Analysis of Sampling from Non-log-concave Densities}

Here we consider the setting where the target density is no longer log-concave. The following theorem states the convergence rate of LS-SGLD in $2$-Wasserstein distance.

\begin{theorem}\label{thm:main_theory}
Under Assumptions \ref{assump:dissipative}, \ref{assump:smooth} and \ref{assump:bound_var}, if set the step size $\eta \le C m\beta^{-1}/M^2$ for some sufficiently small constant $C$, there exist constants $c_0\in[\|\Ab_\sigma\|_2^{-1}, 1]$, $\gamma_1\in[\|\Ab_\sigma\|_2^{-2}, 1]$, $\gamma_2 = d^{-1}\sum_{i=1}^d(1+2\sigma - 2\sigma \cos(2\pi i/d))^{-1}$ and $\bar \Gamma=\big(3/2 + 2(b + \beta^{-1} d)\big)^{1/2}$ such that the output of LS-SGLD satisfies,
\begin{align}\label{eq:bound_nonconvex}
\cW_2\big(\PP(\xb_K),\pi\big)&\le \bar \Gamma (K\eta)^{1/2} \bigg[\bigg(\frac{\gamma_1\beta d\omega^2}{B}K\eta+2\gamma_2 M^2 dK\eta^2\bigg)^{1/2} + \bigg(\frac{\gamma_1\beta d\omega^2}{B}K\eta+2\gamma_2 M^2 dK\eta^2\bigg)^{1/4}\bigg]\notag\\
&\quad+ \big[2\lambda\big(\beta f(0) + \log(\Lambda)\big)\big]^{1/2}\cdot e^{-c_0K\eta/(2\beta \lambda)},
\end{align}
where $\Lambda = \int_{\RR^d} e^{-\beta f(\xb)} \dd\xb$ and $\lambda$ denotes the logarithmic Sobolev constant of the target distribution $\pi\propto e^{-\beta f(\xb)}$.
\end{theorem}


\begin{remark}
The convergence rate of SGLD in $2$-Wasserstein distance can also be obtained from Theorem \ref{thm:main_theory} by setting $\Ab_\sigma = \Ib$, which implies that the constants $c_0,\gamma_1, \gamma_2$ become all $1$'s. It can be verified that the resulting convergence rate matches that proved in \cite{raginsky2017non}.  As a clear comparison, the discretization error induced by both stochastic gradient and numerical integrator of LS-SGLD (the first bracket term of \eqref{eq:bound_nonconvex}) is smaller than that of SGLD, while the ergodicity term of LS-SGLD (the last term of \eqref{eq:bound_nonconvex}) is worse than that of SGLD. Again, we will experimentally demonstrate that the mixing time of LS-LD is not much slower compared with LD, but LS-SGLD can achieve significantly smaller discretization error than SGLD. 
\end{remark}


\section{Numerical Results}\label{section:Experiments}
In this section, we will perform numerical experiments on sampling 2D distributions, training Bayesian Logistic Regression (BLR), and training Convolutional Neural Nets (CNNs). Throughout all the experiments, we regard SGLD \citep{Welling:2011SGLD} and preconditioned SGLD (pSGLD) \citep{Li:2016pSGLD}, which considers local curvature of $f(\bx)$ with RMSProp type of adaptive step size, as benchmarks. In addition, we also incorporated the precondition technique, proposed in \cite{Li:2016pSGLD}, into LS-SGLD, which leads to a variant of LS-SGLD, namely Laplacian smoothing precondictioned SGLD (LS-pSGLD). 


\subsection{Numerical Simulations on Synthetic Dataset}

\subsubsection{2D Gaussian Distribution}
As a simple illustration, we apply the proposed LS-SGLD and LS-pSGLD to sample a 2D Gaussian distribution, studied in \cite{Chen:2014SGHMC}, with the probability density function $e^{f(\bx)} = \exp\left(\frac{1}{2}\bx^T \Sigma^{-1}\bx\right)$ with $\bx \in \mathbb{R}^2$ where $\Sigma=\begin{bmatrix} 
1    & 0.9 \\
0.9  & 1
\end{bmatrix}$. 
We let the prior to be $p(\bx) = \mathcal{N}(0, \nu^2 \Ib)$ with $\nu = 1$. For both SGLD and pSGLD, we let the step size to be $0.19$ which is obtained based on the grid search. For LS-SGLD and LS-pSGLD, we let the Laplacian smoothing parameter $\sigma$ to be $0.1$ with step size to be either $0.19$ or $0.19(1+4\sigma)^{1/4}$. It is worth noting that in 2D $\Ab_\sigma$ becomes
\begin{equation}
\label{Eq:2DLaplacian}
\Ab_\sigma =  \begin{bmatrix} 
1+\sigma & -\sigma \\
-\sigma       & 1+\sigma
\end{bmatrix}.
\end{equation}

\noindent To measure the quality of samples, we consider the MSE between the true and reconstructed covariance matrices; and we calculate the autocorrelation time of the samples to verify the efficacy of different samplers in sampling the correlated distribution above. The autocorrelation time is defined as following
\begin{equation}
\label{Eq:AutocorrelationTime}
\tau = \frac{1}{2} + \sum_{t=1}^\infty \frac{A(t)}{A(0)},
\end{equation}
where $A(t) = \mathbb{E}\left[(\bar{\phi}_\eta - \phi(\bx_0))(\bar{\phi}_\eta - \phi(\bx_t)) \right]$ for any given bounded function $\phi(\bx)$, $\bar{\phi} = \int_\chi \phi(\bx)p(\bx|\mathcal{D})d\bx$ is the population mean of $\phi$, and the empirical mean $\hat{\phi} = \frac{1}{S_T}\sum_{t=1}^T \eta_t\phi(\bx_t)$ with $\eta_t$ being the step size at the $t$-th step and $S_T = \sum_{t=1}^T \eta_t$.
\medskip

\noindent Figure~\ref{Fig:Gaussian:Low-Condition} (a) and (c) plot the first $600$ samples from the target distribution by different samplers. We use the same step size $0.19$ for all the four samplers in the experiments shown in Fig.~\ref{Fig:Gaussian:Low-Condition} (a), and use a larger step size $0.19(1+4\sigma)^{1/4}$ for LS-SGLD and LS-pSGLD in experiments shown in Fig.~\ref{Fig:Gaussian:Low-Condition} (c). Qualitatively, Laplacian smoothing can enhance the quality of samples, and the improvement becomes more remarkable when we use a larger step size. Next, we draw $2\times 10^5$ samples from the target distribution by different samplers and we use these samples to reconstruct the covariance matrix. For SGLD and pSGLD, we use a set of step size $\{0.19, 0.19\times 0.8, 0.19\times 0.8^2, 0.19\times 0.8^3, 0.19\times 0.8^4\}$. For LS-SGLD and LS-pSGLD we test two sets of step sizes: (i) $\{0.19, 0.19\times 0.8, 0.19\times 0.8^2, 0.19\times 0.8^3, 0.19\times 0.8^4\}$, (ii) $\{0.19, 0.19\times 0.8, 0.19\times 0.8^2, 0.19\times 0.8^3, 0.19\times 0.8^4\} \times (1+4\sigma)^{1/4}$. Figure~\ref{Fig:Gaussian:Low-Condition} (b) and (d) plot the autocorrelation time v.s. reconstruction error of the covariance matrix. In (b) we use the same set of step sizes for all the four samplers, and in (d) we use a larger step size for LS-SGLD and LS-pSGLD. We see that reconstruction errors can be reduced significantly when Laplacian smoothing is used. Moreover, Laplacian smoothing can also reduce the autocorrelation time in pSGLD.

\begin{figure}[!ht]
\centering
\begin{tabular}{cc}
\includegraphics[width=0.59\columnwidth]{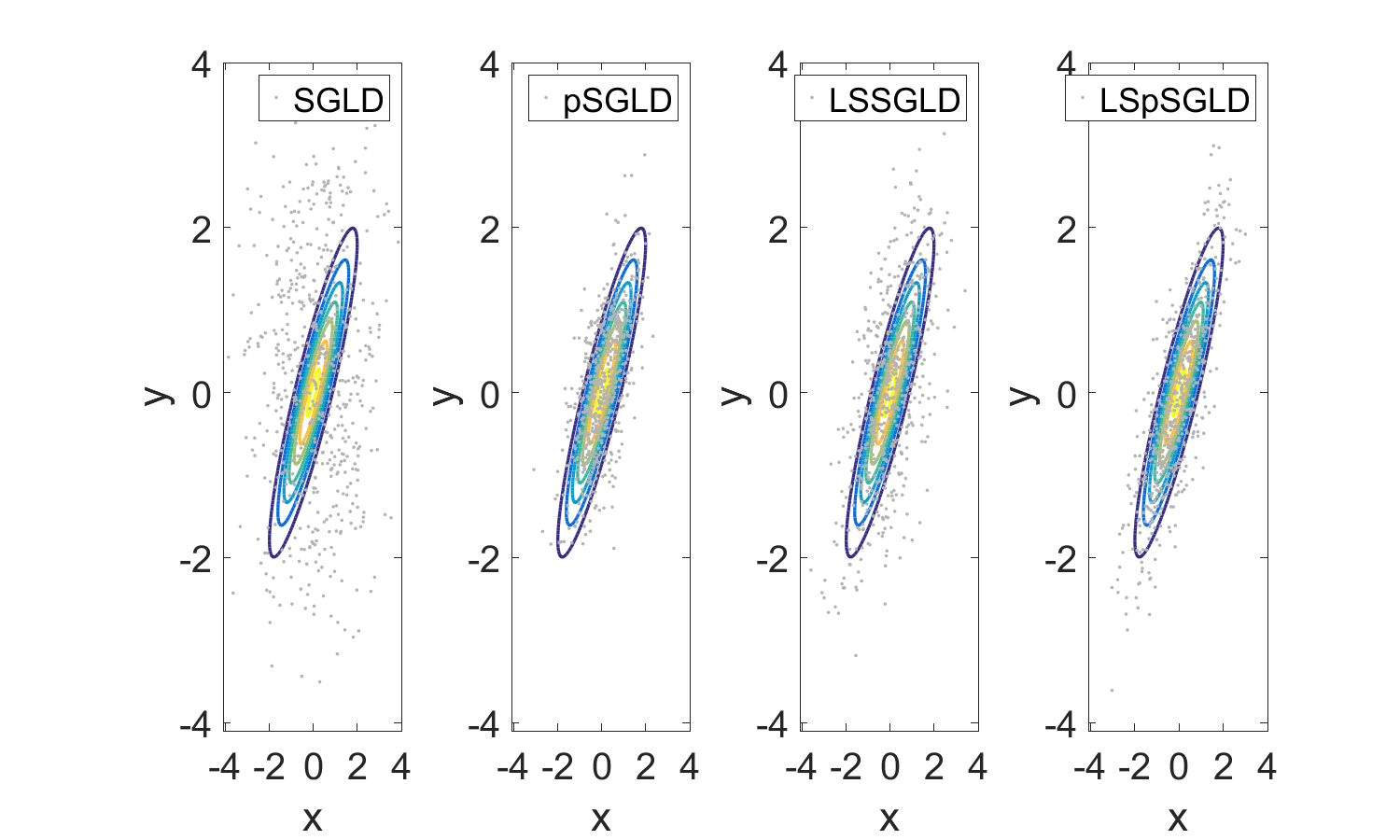}&
\includegraphics[width=0.35\columnwidth]{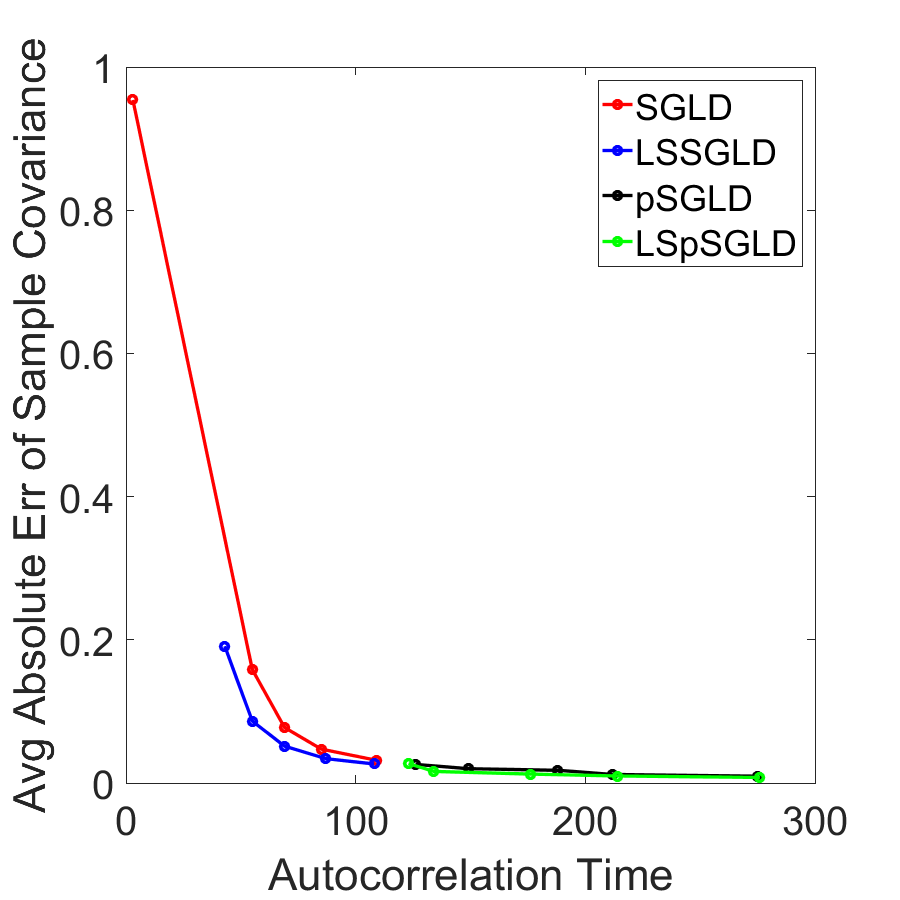}\\
(a) Samples & (b) Error v.s. ACT\\
\includegraphics[width=0.59\columnwidth]{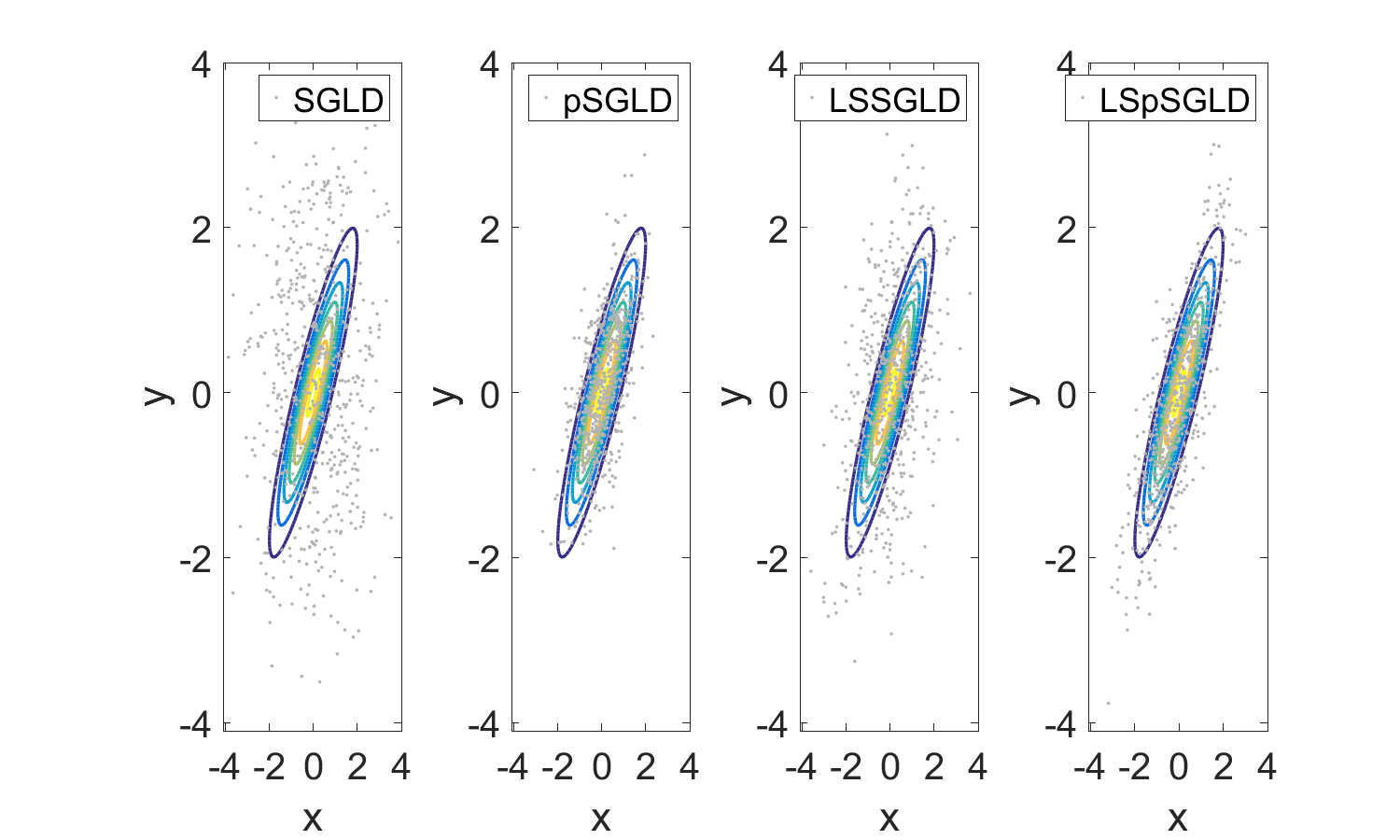}&
\includegraphics[width=0.35\columnwidth]{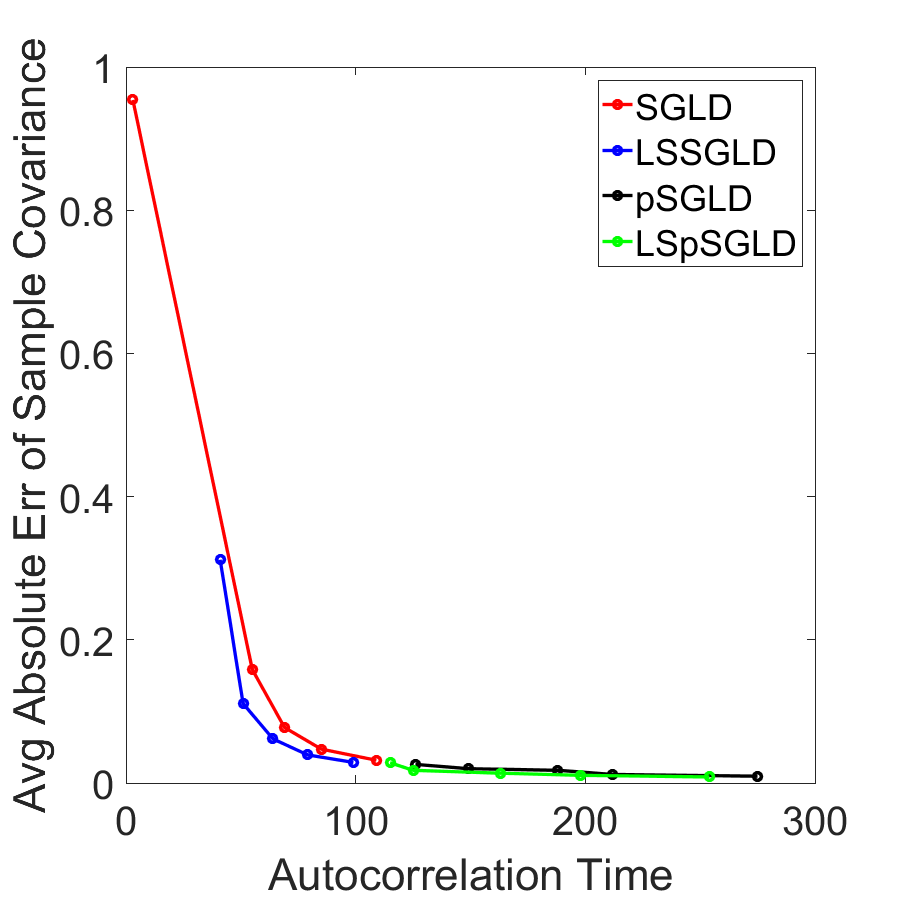}\\
(c) Samples & (d) Error v.s. ACT\\
\end{tabular}
\caption{Contrasting sampling of a 2D Gaussian distribution with covariance matrix $\Sigma$, with $\Sigma_{11}=\Sigma_{22}=1$ and $\Sigma_{12}=\Sigma_{21}=0.9$, using different samplers. (a) and (c): the first $600$ samples draw by SGLD, pSGLD. (b) and (d): LS-SGLD, and LS-pSGLD. In (c) and (d), we multiply the step size for LS-SGLD and LS-pSGLD by a factor $(1+4\sigma)^{1/4}$.}
\label{Fig:Gaussian:Low-Condition}
\end{figure}

\subsubsection{2D Gaussian Mixture Distribution}\label{subsection:GaussianMixture}
In this subsection, we compare the performance of SGLD, pSGLD, LS-SGLD and LS-pSGLD on a Gaussian mixture distribution. In particular, we consider the target distribution 
$$
\pi \propto \exp{\left(-f(\bx)\right)} = \exp{\left( -\frac{1}{n}\sum_{i=1}^nf_i(\bx) \right)},\ \ n = 500,
$$
where each component $\exp{\left(-f_i(\bx)\right)}$ is defined as
$$
\exp{\left(-f_i(\bx)\right)} = \frac{2}{3}e^{-\frac{\|\bx - \va_i\|_2^2}{2}} + \frac{1}{3}e^{-\frac{\|\bx + \va_i\|_2^2}{2}},
$$
where we sample $\va_i$ by the MCMC sampler with MH correction from the following 2D Gaussian distribution 
$$
\mathcal{N}\left(\begin{bmatrix} 
2  \\
2 
\end{bmatrix}, \begin{bmatrix} 
2 & 0 \\
0 & 2
\end{bmatrix}\right).
$$

\noindent The function $f_i(\bx)$ and its gradient can be simplified as
$$
f_i(\bx) = \frac{\|\bx - \va_i\|_2^2}{2} - \log\left( \frac{2}{3}+\frac{1}{3}\exp{\left(-2\la\va_i, \bx\ra\right)} \right),
$$
$$
\nabla f_i(\bx) = \bx - \va_i + \frac{2\va_i}{2+ \exp\left(2 \la\bx, \va_i \ra\right) }.
$$
It can be easily verified that if $\va_i$ satisfies $\|\va_i\|_2> 3/2$, the function $f_i(\bx)$ defined above is nonconvex. Moreover, it can be seen that
$$
\la\nabla f_i(\bx), \bx\ra = \|\bx\|_2^2 - \frac{ \exp\left(2 \la\bx, \va_i \ra\right)}{2 + \exp\left(2 \la\bx, \va_i \ra\right)}\la\va_i, \bx\ra \geq \frac{1}{2}\|\bx\|_2^2 - \frac{1}{2}\|\va_i\|_2^2,
$$
which suggests that the function $f_i(\bx)$ satisfies the Dissipative Assumption~\ref{assump:dissipative} with $m=\frac{1}{2}$ and $b=\frac{\|\va_i\|_2^2}{2}$, and it further implies that $f(\bx)$ is also dissipative. 
\medskip

\noindent Since it takes a large number of samples to characterize the distribution, which makes repeated experiments computationally expensive, we instead follow \citet{bardenet2017markov, zou2019sampling, zou2019stochastic} to use iterates along one Markov chain to visualize the distribution of iterates obtained by MCMC algorithms. We run the four samplers with different numbers of iteration where we set the batch size to be $10$. We plot the distributions generated by different samplers with different numbers of iterations in Fig.~\ref{Fig:Binormal:Samples}. As shown in Fig.~\ref{Fig:Binormal:Samples} (c), (f), and (u), when the number of iterations is large enough, e.g. $10^6$, the sample distributions of all the three samplers matches well with the reference distribution (sampled by ground-truth sampler, e.g., MCMC with MH step). However, when the number of iterations is not enough, there is a large discrepancy between the sample and target distributions, as shown in Fig.~\ref{Fig:Binormal:Samples} (a), (d), and (g). With a moderate number of iterations, say $5\times 10^5$, the sample distribution from LS-SGLD is better than the other two (Fig.~\ref{Fig:Binormal:Samples} (b), (e), and (d)).
\medskip

\begin{figure}[!ht]
\centering
\begin{tabular}{ccc}
\includegraphics[width=0.3\columnwidth]{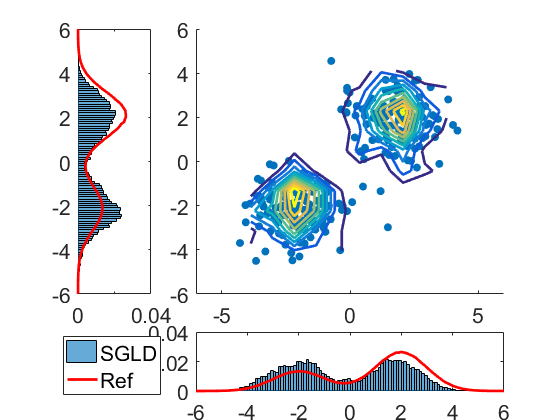}&
\includegraphics[width=0.3\columnwidth]{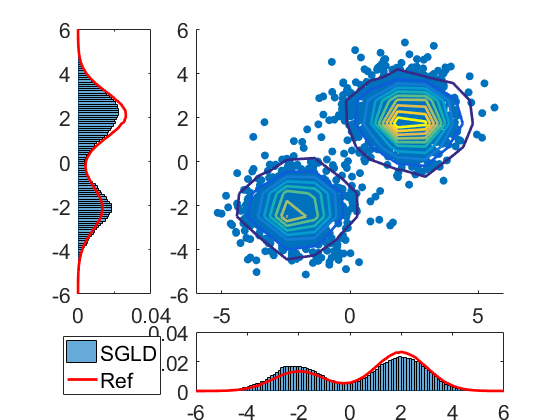}&
\includegraphics[width=0.3\columnwidth]{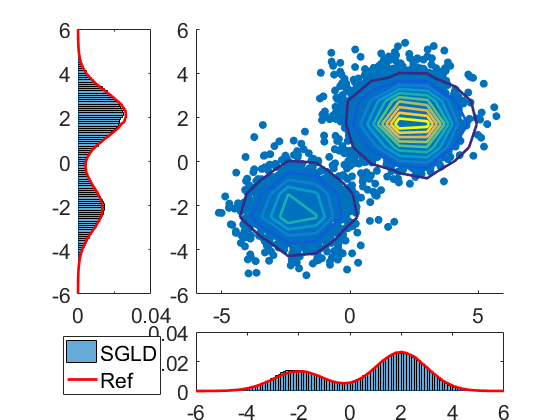}\\
(a) SGLD (1E5)  & (b) SGLD (5E5) & (c) SGLD (1E6)\\
\includegraphics[width=0.3\columnwidth]{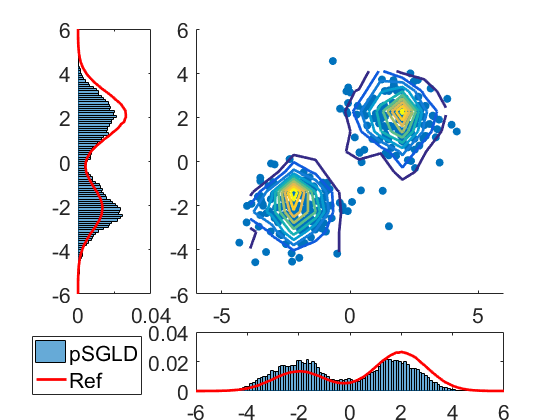}&
\includegraphics[width=0.3\columnwidth]{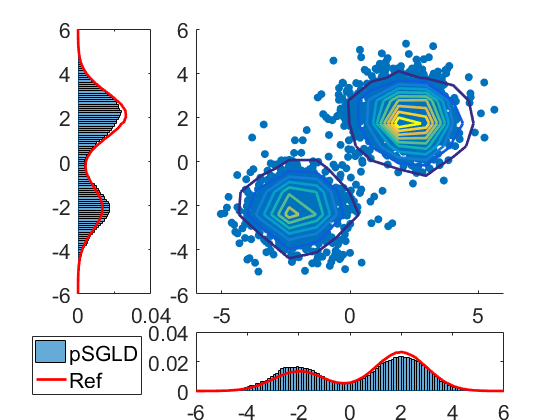}&
\includegraphics[width=0.3\columnwidth]{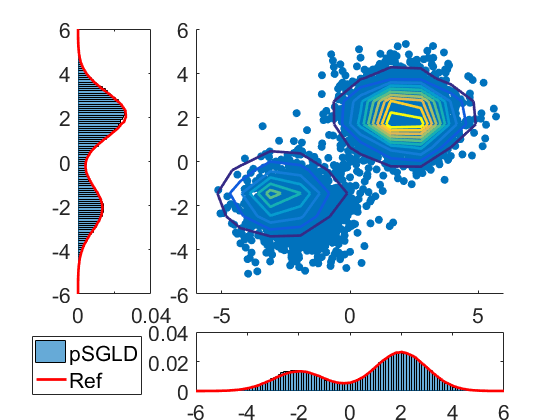}\\
(d) pSGLD (1E5) & (e) pSGLD (5E5) & (f) pSGLD (1E6)\\
\includegraphics[width=0.3\columnwidth]{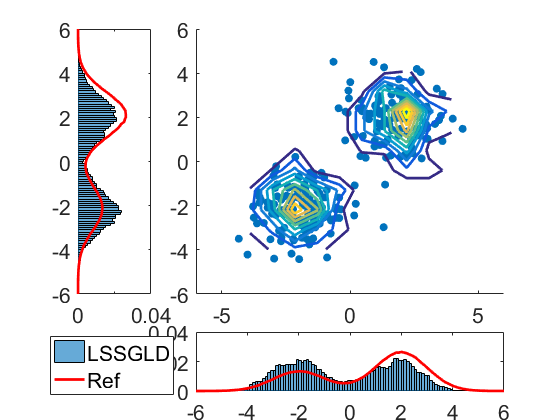}&
\includegraphics[width=0.3\columnwidth]{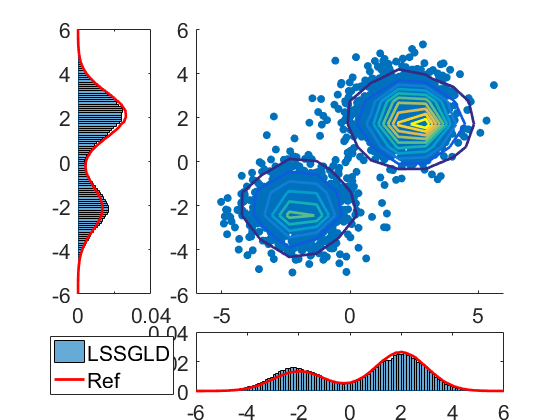}&
\includegraphics[width=0.3\columnwidth]{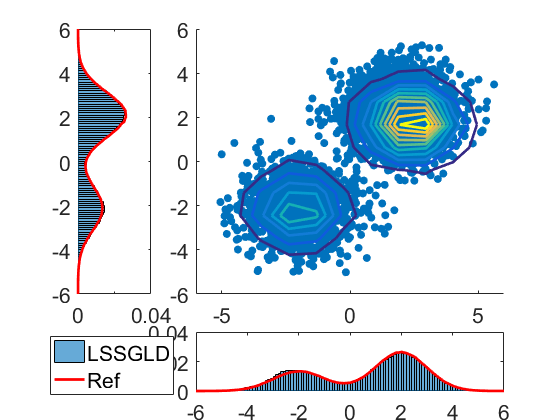}\\
(g) LS-SGLD (1E5) & (h) LS-SGLD (5E5) & (i) LS-SGLD (1E6)\\
\end{tabular}
\caption{Kernel density plots of samples generated from Gaussian mixture distribution using SGLD, LS-SGLD, pSGLD and LS-pSGLD. We set $\sigma=1.0$ for LS-SGLD and LS-pSGLD.}
\label{Fig:Binormal:Samples}
\end{figure}

\noindent Let us further evaluate the sample quality in a quantitative approach. We first apply the MCMC with MH step to sample $10K$ samples from the above target distribution. Then we apply SGLD, pSGLD, and LS-SGLD to sample different numbers of samples, respectively, from the target distribution. We measure the 2-Wasserstein distance between the last $10K$ samples of the different number of samples by the above three stochastic gradient samplers with the MH samples. We list the Wasserstein distance between the last $10K$ samples of different numbers of samples from different samplers with the MH samples in Table~\ref{Table:Wasserstein-Distance}. These results show that the samples generated by LS-SGLD are consistently closer to the samples from MCMC with MH correction.

\begin{table}[!ht]
\centering
\fontsize{10.}{10.}\selectfont
\begin{threeparttable}
\caption{2-Wasserstein distance between samples sampled by MCMC with Metropolis–Hastings correction and different stochastic gradient Langevin dynamics.}\label{Table:Wasserstein-Distance}
\begin{tabular}{cccc}
\toprule[1.0pt]
\ \ \ \ \ \ \ \ \ \ \ \# of Samples\ \ \ \ \ \ \ \ \ \ \   &\ \ \ \ \ \ \ \ \ \ \  1E5\ \ \ \ \ \ \ \ \ \ \   &\ \ \ \ \ \ \ \ \ \ \  5E5\ \ \ \ \ \ \ \ \ \ \   &\ \ \ \ \ \ \ \ \ \ \   9E5\ \ \ \ \ \ \ \ \ \ \  \cr
\midrule[0.8pt]
SGLD    & 0.695 & 6.726 & 0.285 \cr
pSGLD   & 5.364 & 0.286 & 6.728 \cr
LS-SGLD  & 0.421 & 0.414 & 0.418 \cr
\bottomrule[1.0pt]
\end{tabular}
\end{threeparttable}
\end{table}

\subsubsection{Comparison of the mixing time between LD and LS-LD}
\noindent To verify that Laplacian smoothing does not slow down the mixing rate of the continuous-time Markov process, we conduct the following experiments. First, we apply the MCMC with Metropolis-Hasting correction step to sample $10$K points, respectively, from the following two distributions.
\begin{itemize}
\item Gaussian distribution with the following probability density function
\begin{equation}\label{Eq:PDF}
p(x, y) = \frac{1}{9\pi}\exp{\left(-\left(\frac{(x-1)^2}{3^2} + \frac{(y-2)^2}{3^2}\right)\right)}.
\end{equation}

\item The Gaussian mixture distribution described in subsection~\ref{subsection:GaussianMixture}.
\end{itemize}
Second, we use either LD or LS-LD (which can be approximated by Euler-Maruyama discretization with very small step size), to draw samples from the above two distributions and use these samples to estimate the mean of the target densities. Figure~\ref{Fig:Reconstruction:MSE} plots the MSE between the true and reconstructed (from a different number of samples) means, and they show that LD and LS-LD perform similarly in reconstructing the mean of the target densities.


\begin{figure}[!ht]
\centering
\begin{tabular}{cc}
\includegraphics[width=0.45\columnwidth]{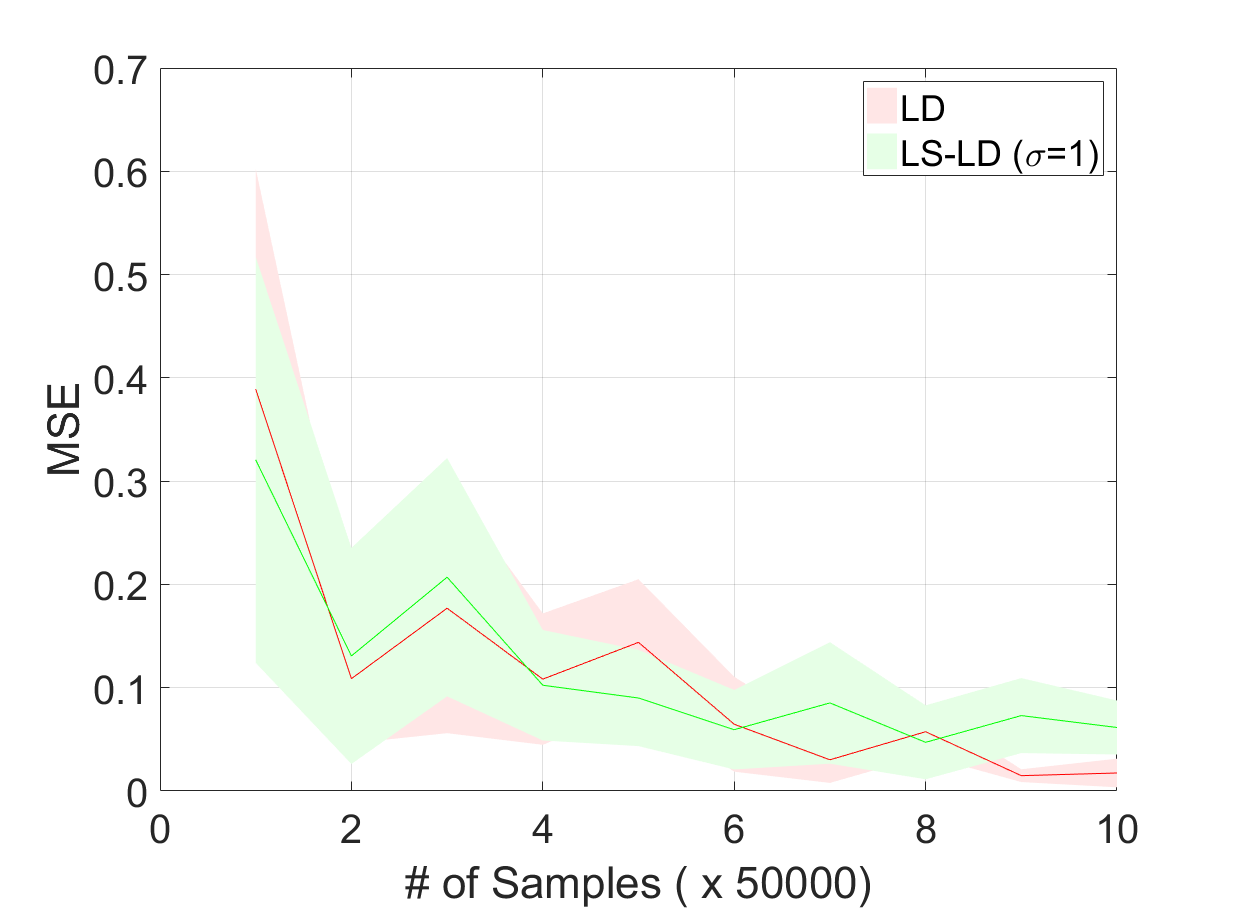}&
\includegraphics[width=0.45\columnwidth]{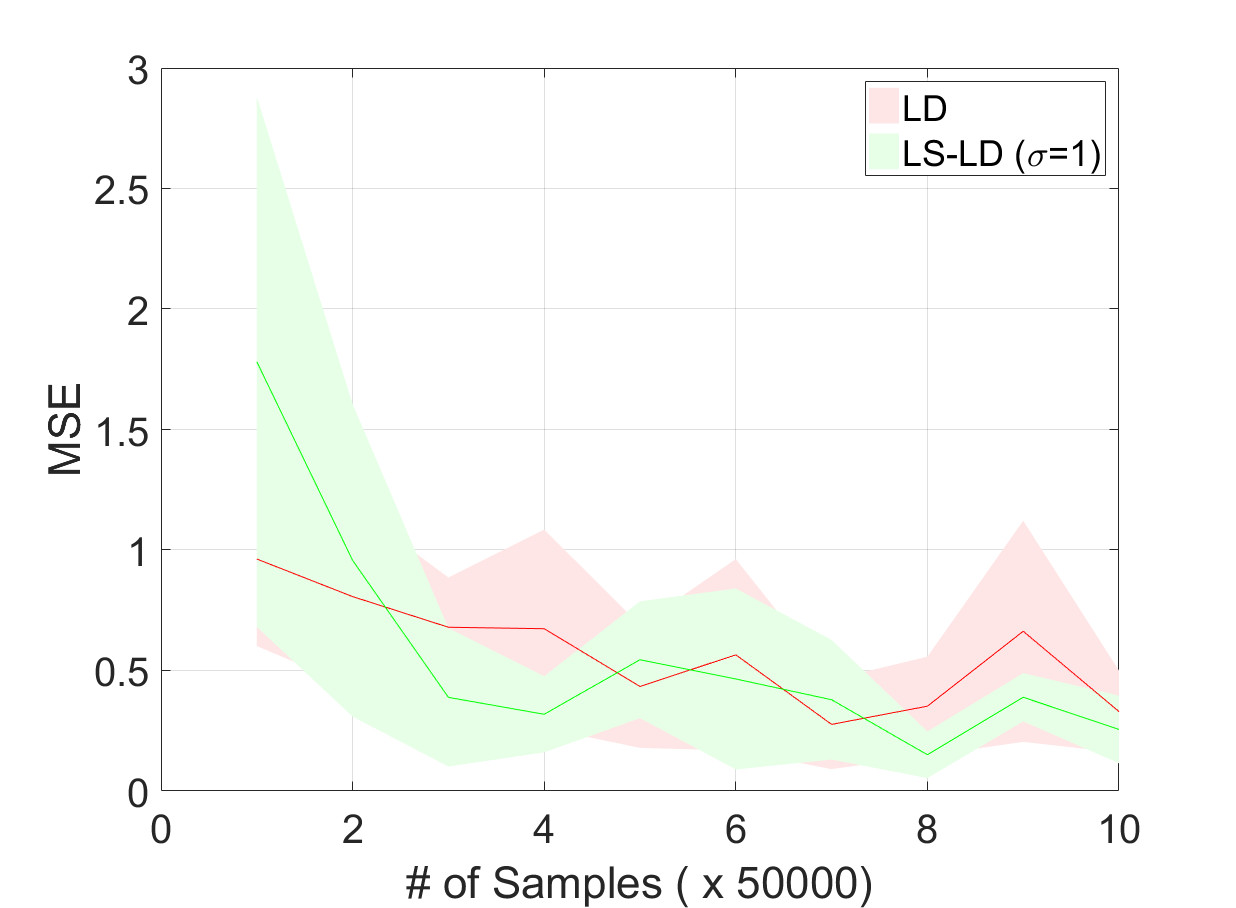}\\
\end{tabular}
\caption{MSE between the true and reconstructed means from different numbers of samples generated by LD and LS-LD (10 independent runs).
}
\label{Fig:Reconstruction:MSE}
\end{figure}

\subsection{Bayesian Logistic Regression}
Suppose we observe $n$ i.i.d samples $\{\bd_i, y_i\}_{i=1, 2, \cdots, n}$ where $\bd_i\in \RR^d$ and $y_i \in \{-1, 1\}$ denote the feature and the corresponding label of the $i$-th sample instance. The likelihood of BLR model is given by
$$
p(y_i|\bd_i, \bx) = \frac{1}{1+\exp{\left(-y_i\la\bd_i, \bx \ra \right)}}
$$
where $\bx$ is the parameter to be learned. We use a Gamma prior $p(\bx) \propto \|\bx\|_2^{-\lambda}\exp{\left(-\theta \|\bx\|_2\right)}$ with $\lambda=1$ and $\theta = 10^{-2}$. Then we formulate the logarithmic posterior distribution as follows
$$
\log\left[ p(\bx| \bd_1, \bd_2, \cdots, \bd_n; y_1, y_2, \cdots, y_n) \right] \propto -\frac{1}{n}\sum_{i=1}^n f_i(\bx),
$$
where $f_i(\bx) = n\log\left(1+e^{-y_i\la\bd_i, \bx\ra} \right) + \lambda\log\left(\|\bx\|_2\right) + \theta \|\bx\|_2$.
\medskip

\noindent We use SGLD, pSGLD, LS-SGLD, and LS-pSGLD with batch size $5$ to train a Bayesian logistic regression (BLR) model on the benchmark {\it a3a} dataset from the UCI machine learning repository \footnote{\url{https://archive.ics.uci.edu/ml/index.php}}. The {\it a3a} dataset contains $3185$ training data and $29376$ test instances, each data instance is of dimension $122$.
We use the grid search to determine the optimal learning rate for SGLD ($0.001$) and pSGLD ($0.002$), and then we multiply them by $(1+4\sigma)^{1/4}$ to get the learning rate for LS-SGLD and LS-pSGLD. We set the burn-in to be $1000$ for all these four samplers. After burn-in, we compute the moving average of the sample parameters to estimate the regression parameters $\bx$. 
We plot iteration v.s. negative log-likelihood and accuracy in Fig.~\ref{Fig:BLR}, 
and we see that Laplacian smoothing reduces the negative log-likelihood and increases the accuracy. The preconditioning accelerates mixing initially, however, the gap between sampling and target distribution is remarkably larger than the case without preconditioning.
\begin{figure}[!ht]
\centering
\begin{tabular}{cc}
\includegraphics[width=0.45\columnwidth]{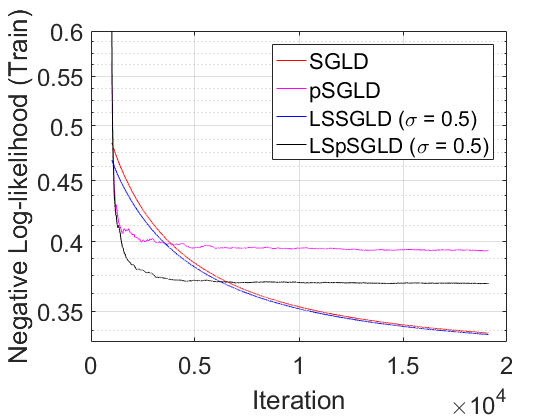}&
\includegraphics[width=0.45\columnwidth]{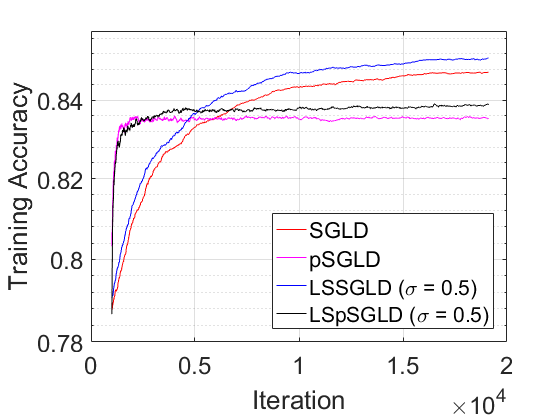}\\
(a) Log-likelihood (Training Set)  & (b) Training Accuracy \\
\includegraphics[width=0.45\columnwidth]{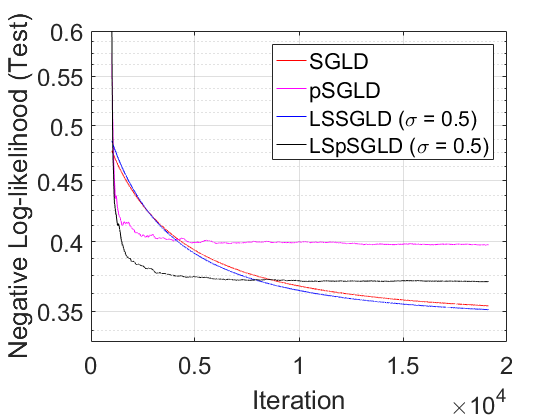}&
\includegraphics[width=0.45\columnwidth]{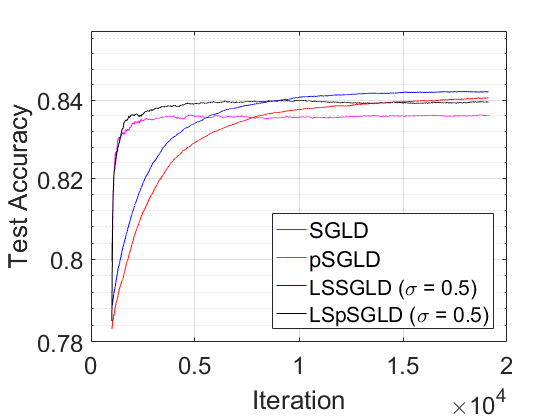}\\
(c) Log-likelihood (Test Set) & (d) Test Accuracy \\
\end{tabular}
\caption{Convergence comparison for Bayesian logistic regression, where the X-axis represents the number of iterations and Y-axis represents the negative log-likelihood/accuracy. (a) Negative log-likelihood on training dataset; (b) Accuracy on training dataset; (c) Negative log-likelihood on test dataset; (d) Accuracy on test dataset. 
}
\label{Fig:BLR}
\end{figure}


\subsubsection{Variance reduction in stochastic gradient}
We numerically verify the efficiency of variance reduction on BLR for {\it a3a} dataset classification. We first compute a path by full batch SGLD with the same learning rate as before, and meanwhile, we record the Laplacian smoothing gradient on each point along the path. Then we compute the Laplacian smoothing stochastic gradients on each point along the path by using different batch size and $\sigma$. We run 100 independent experiments to acquire the Laplacian smoothing stochastic gradients, and then we compute the variance of these stochastic gradients by regarding the full batch Laplacian smoothing gradient as the mean. In Table~\ref{Table:Variance}, we report the maximum variance, among all coordinates of the gradient and all points on the descent path, for each pair of batch size and $\sigma$.

\begin{table*}[!ht]
\centering
\fontsize{10.0}{10}\selectfont
\begin{threeparttable}
\caption{
The maximum variance of the stochastic gradients generated by LS-SGLD on training BLR on the {\it a3a} data.
$\sigma=0$ reduces to SGLD.}
\begin{tabular}{cccc}
\toprule
\ \ \ \ \ \ \ \ \ Batch Size\ \ \ \ \ \ \ \ \      &\ \ \ \ \ \ \ \ \ \ \ 10\ \ \ \ \ \ \ \ \ \ \   &\ \ \ \ \ \ \ \ \ \ \ 15\ \ \ \ \ \ \ \ \ \ \   &\ \ \ \ \ \ \ \ \ \ \ 50\ \ \ \ \ \ \ \ \ \ \  \cr
\midrule
$\sigma = 0$     & 7.69E-1 & 3.17E-1 & 5.69E-2\cr
$\sigma = 0.5$   & 2.56E-1 & 1.06E-2 & 1.96E-2 \cr
$\sigma = 1.0$   & 1.54E-1 & 6.37E-2 & 1.21E-2 \cr
$\sigma = 2.0$   & 8.52E-2 & 3.54E-2 & 7.04E-3 \cr
\bottomrule
\end{tabular}
\label{Table:Variance}
\end{threeparttable}
\end{table*}

\subsection{Bayesian Convolutional Neural Network}
We consider training a CNN by SGLD, pSGLD, LS-SGLD, and LS-pSGLD on the MNIST benchmark with batch size $100$, the architecture of the CNN is
\begin{eqnarray*}
\mbox{CNN:}\  {\rm input}_{28 \times 28} \rightarrow {\rm conv}_{20, 5, 2} \rightarrow
{\rm conv}_{20, 20, 5} \rightarrow {\rm fc}_{128} \rightarrow {\rm softmax}.
\end{eqnarray*}
The notation ${\rm conv}_{c, k, m}$ denotes a 2D convolutional layer with $c$ output channels, each of which is the sum of a channel-wise convolution operation on the input using a learnable kernel of size $k \times k$, it further adds ReLU nonlinearity and max-pooling with stride size $m$. ${\rm fc}_{128}$ is an affine transformation that transforms the input to a vector of dimension $128$. Finally, the tensors are activated by a multi-class logistic function.
\medskip

\noindent 
Similar to BLR, we use a Gamma prior $p(\bx) \propto \|\bx\|_2^{-\lambda}\exp{\left(-\theta \|\bx\|_2\right)}$ with $\lambda=1$ and $\theta = 5e^{-4}$. Again, we use the grid search to find the optimal step size for SGLD and pSGLD which is $0.02$ and $2e-4$, respectively. 
We multiply the optimal step size for SGLD and pSGLD by a factor $(1+4\sigma)^{1/4}$ to get the step size for LS-SGLD and LS-pSGLD, and we let $\sigma=0.5$ for Laplacian smoothing. The comparisons between different sampling algorithms are plotted in Fig.~\ref{Fig:CNN}, we see that Laplacian smoothing reduces the negative log-likelihood and increases the accuracy of both training and test datasets. The preconditioning accelerates mixing and reduces the gap between sampling and target distribution.
Here, we applied early stopping in training CNN by pSGLD and LS-pSGLD.
\begin{figure}[!ht]
\centering
\begin{tabular}{cc}
\includegraphics[width=0.45\columnwidth]{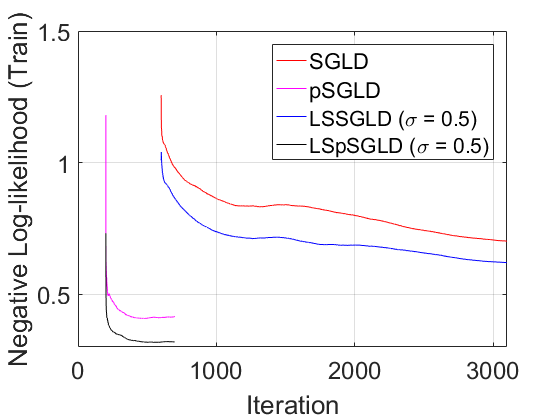}&
\includegraphics[width=0.45\columnwidth]{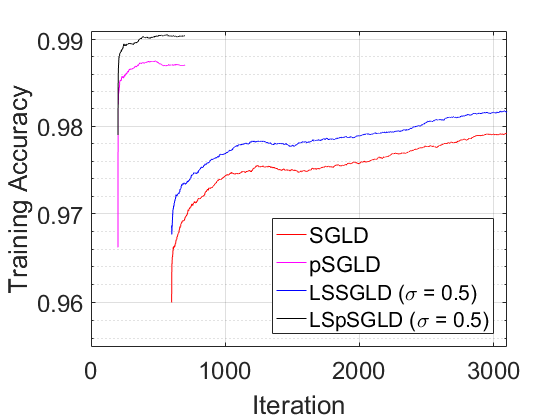}\\
(a) Log-likelihood (Training Set)  & (b) Training Accuracy \\
\includegraphics[width=0.45\columnwidth]{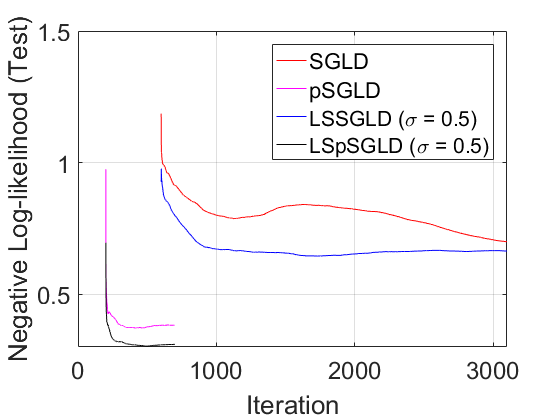}&
\includegraphics[width=0.45\columnwidth]{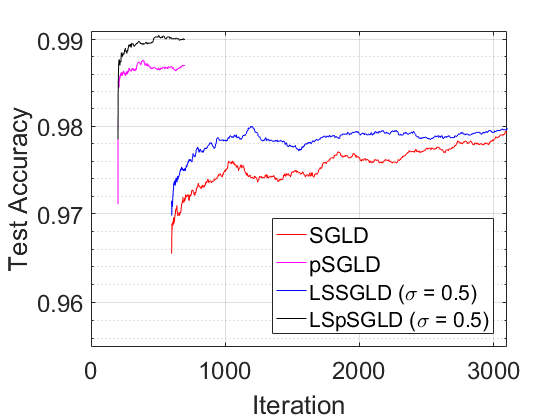}\\
(c) Log-likelihood (Test Set) & (d) Test Accuracy \\
\end{tabular}
\caption{
Convergence comparison for Bayesian convolutional neural network, where the X-axis represents the number of iterations and  Y-axis represents the negative log-likelihood/accuracy. (a) Negative log-likelihood on training dataset; (b) Accuracy on training dataset; (c) Negative log-likelihood on test dataset; (d) Accuracy on test dataset.
}
\label{Fig:CNN}
\end{figure}


\section{Conclusions}\label{section:Conclusion}
In this paper, we integrate Laplacian smoothing with Stochastic Gradient Langevin Dynamics (SGLD) to reduce the gap between the sample and target distributions. The resulting algorithm also allows us to take a larger step size. The proposed algorithm is simple to implement and the extra computation and memory costs compared with the SGLD are negligible when the Fast Fourier Transform (FFT)-based algorithms is employed to resolve the dynamics of the resulting Laplacian Smoothing SGLD (LS-SGLD). We show, both theoretically and empirically, that LS-SGLD can improve the sample quality. It is straightforward to extend Laplacian smoothing to the other Markov Chain Monte Carlo (MCMC) algorithms, e.g., the stochastic gradient Hamiltonian Monte Carlo \citep{Chen:2014SGHMC}. 



\section*{Acknowledgments}
This material is based on research sponsored by the National Science Foundation under grant number DMS-1924935 and DMS-1554564 (STROBE). The Air Force Research Laboratory under grant numbers FA9550-18-0167 and MURI FA9550-18-1-0502, the Office of Naval Research under grant number N00014-18-1-2527. QG is partially supported by the National Science Foundation under grant number SaTC-1717950.

\newpage
\appendix
\section{Missing proof in Section \ref{section:Algorithm}}
In this section, we provide the proof of Proposition \ref{Prop:StationaryDistribution}.

\begin{proof}[Proof of Proposition \ref{Prop:StationaryDistribution} ]
Let $p(\xb, t)$ be the distribution of $\bX_t$. Then we know that $p(\xb, t)$ satisfies the following Fokker-Planck equation
\begin{align}\label{eq:fokker_planck}
\frac{\partial p(\xb, t)}{\partial t} &= \frac{1}{\beta}\la\Ab_\sigma^{-1}, \nabla^2 p(\xb,t)\ra + \la\nabla, p(\xb, t)\Ab_\sigma^{-1} \nabla f(\xb)\ra\notag\\
& = \frac{1}{\beta}\la\nabla, \Ab_\sigma^{-1}\nabla p(\xb,t)\ra + \la\nabla, p(\xb, t)\Ab_\sigma^{-1} \nabla f(\xb)\ra,
\end{align}
where $\la\nabla, \mathbf{h}(\xb)\ra$ denotes the divergence of the vector field $\mathbf{h}(\xb)$.
Since the stationary distribution $\pi$ satisfies $\partial \pi/\partial t = 0$, we have
\begin{align*}
\frac{1}{\beta}\la\nabla, \Ab_\sigma^{-1}\nabla p(\xb,t)\ra + \la\nabla, p(\xb, t)\Ab_\sigma^{-1} \nabla f(\xb)\ra = 0,
\end{align*}
which further implies that $\beta^{-1} \nabla \pi + \pi\nabla f(\xb)=0$. Solving this equation directly gives $\pi \propto e^{-\beta f(\xb)}$, which completes the proof.

\end{proof}

\section{Proof of Main Theory}\label{sec:proof_main}

In order to bound sampling error between the distribution of the output of LS-SGLD and the target distribution $\pi\propto e^{-\beta f(\xb)}$, we consider a reference sequence generated by LS-LD  \eqref{eq:langevin_cont}, denoted by $\{\bX_t\}_{t\ge 0}$. Let $\bX_0 = \xb_0$, by triangle inequality, we can decompose the $2$-Wasserstein distance $\cW_2(\PP(\xb_K), \pi)$ as follows
\begin{align*}
\cW_2(\PP(\xb_K), \pi)\le \cW_2(\PP(\xb_K), \PP(\bX_{K\eta})) + \cW_2(\PP(\bX_{K\eta}), \pi).
\end{align*}
The first term on the R.H.S. stands for the discretization error of the numerical integrator, and the second term denotes the ergodicity of LS-LD \eqref{eq:langevin_cont}, which characterizes the mixing time of LS-LD.
In what follows, we first deliver the following lemma that characterizes the error term $\cW_2(\PP(\bX_{K\eta}), \pi)$. 
\begin{lemma}\label{lemma:ergodicity}
Under Assumptions \ref{assump:dissipative} and \ref{assump:smooth}, there exists a constant $c_0\in[\|\Ab_\sigma\|_2^{-1}, 1]$
\begin{align*}
\cW_2\big(\PP(\bX_{K\eta}), \pi\big)\le \big[2\lambda\big(\beta f(0) + \log(\Lambda)\big)\big]^{1/2}\cdot e^{-c_0K\eta/(2\beta \lambda)},
\end{align*}
where $\Lambda = \int_{\RR^d} e^{-\beta f(\xb)} \dd\xb$ and $\lambda$ denotes the logarithmic Sobolev constant of the target distribution $\pi\propto e^{-\beta f(\xb)}$.
\end{lemma}

Note that Lemma \ref{lemma:ergodicity} does not require that the target density is log-concave, which can be utilized to prove the convergence rate of LS-SGLD for sampling both log-concave and non-log-concave densities. In the following, we are going to complete the proofs of Theorems \ref{thm:main_theory_convex} and \ref{thm:main_theory}. 

\subsection{Proof of Theorem \ref{thm:main_theory_convex}}
We first provide the following lemma which proves an upper bound of the discretization error $\cW_2\big(\PP(\xb_K), P(\bX_{K\eta})\big)$ for sampling log-concave densities.
\begin{lemma}\label{lemma:discrete_error_convex}
Under Assumptions \ref{assump:dissipative}, \ref{assump:smooth}, \ref{assump:bound_var} and \ref{assump:convex}, if set the step size $\eta \le C m\beta^{-1}/M^2$ for some sufficiently small constant $C$, there exist constants $\gamma_1\in[\|\Ab_\sigma\|_2^{-2},1]$ and $\gamma_2 = d^{-1}\sum_{i=1}^d(1+2\sigma - 2\sigma \cos(2\pi i/d))^{-1} $ such that the following holds
\begin{align*}
\cW_2\big(\PP(\xb_K), P(\bX_{K\eta})\big)\le \bigg(\frac{2\gamma_1d\omega^2K\eta^2}{B}\bigg)^{1/2} +\big[8\gamma_2K(K+1) \beta^{-1} d\eta ^3\big]^{1/2}
\end{align*}
\end{lemma}
Then we can complete the proof of Theorem \ref{thm:main_theory_convex} as follows.
\begin{proof}[Proof of Theorem \ref{thm:main_theory_convex}]
By triangle inequality and Lemmas \ref{lemma:ergodicity} and \ref{lemma:discrete_error_convex}, it is evident that
\begin{align*}
\cW_2(\PP(\xb_K), \pi)&\le \cW_2(\PP(\xb_K), \PP(\bX_{K\eta})) + \cW_2(\PP(\bX_{K\eta}), \pi)\notag\\
&\le \bigg(\frac{2\gamma_1 d\omega^2K\eta^2}{B}\bigg)^{1/2} + \big[ 8\gamma_2 K(K+1) \beta^{-1} d\eta^3 \big]^{1/2} + \big[2\lambda\big(\beta f(0) + \log(\Lambda)\big)\big]^{1/2}\cdot e^{-c_0K\eta/(2\beta \lambda)},
\end{align*}
which completes the proof.

\end{proof}

\subsection{Proof of Theorem \ref{thm:main_theory}}
Similar to the proof of Theorem \ref{thm:main_theory_convex}, we provide the following lemma that characterizes the discretization error $\cW_2\big(\PP(\xb_k), \PP(\bX_{k\eta})\big)$ for sampling from non-log-concave densities.
\begin{lemma}\label{lemma:discretization}
Under Assumptions \ref{assump:dissipative} and \ref{assump:smooth}, if set the step size $\eta \le C m\beta^{-1}/M^2$ for some sufficiently small constant $C$, there exist constants $\gamma_1\in[\|\Ab_\sigma\|_2^{-2}, 1]$, $\gamma_2 = d^{-1}\sum_{i=1}^d(1+2\sigma - 2\sigma \cos(2\pi i/d))^{-1}$ and $\bar \Gamma=\big(3/2 + 2(b + \beta^{-1} d)\big)^{1/2}$ such that the following holds,
\begin{align*}
\cW_2\big(\PP(\xb_K),\PP(\bX_{K\eta})\big)\le \bar \Gamma (K\eta)^{1/2} \bigg[\bigg(\frac{\gamma_1\beta d\omega^2}{2B}K\eta+2\gamma_2\beta M^2 dK\eta^2\bigg)^{1/2} + \bigg(\frac{\gamma_1\beta d\omega^2}{2B}K\eta+2\gamma_2\beta M^2 dK\eta^2\bigg)^{1/4}\bigg].  
\end{align*}
\end{lemma}

\begin{proof}[Proof of Theorem \ref{thm:main_theory}]
By triangle inequality and Lemmas \ref{lemma:ergodicity} and \ref{lemma:discretization}, it is evident that
\begin{align*}
\cW_2(\PP(\xb_K), \pi)&\le \cW_2(\PP(\xb_K), \PP(\bX_{K\eta})) + \cW_2(\PP(\bX_{K\eta}), \pi)\notag\\
&\le \bar \Gamma (K\eta)^{1/2} \bigg[\bigg(\frac{\gamma_1\beta d\omega^2}{2B}K\eta+2\gamma_2 M^2 dK\eta^2\bigg)^{1/2} + \bigg(\frac{\gamma_1\beta d\omega^2}{2B}K\eta+2\gamma_2 M^2 dK\eta^2\bigg)^{1/4}\bigg]\notag\\
&\quad+ \big[2\lambda\big(\beta f(0) + \log(\Lambda)\big)\big]^{1/2}\cdot e^{-c_0K\eta/(2\beta \lambda)},
\end{align*}
which completes the proof.

\end{proof}

\section{Proof of lemmas in Appendix \ref{sec:proof_main}}
\label{sec:proof_lemma}

\subsection{Proof of Lemma \ref{lemma:ergodicity}}
In order to prove Lemma \ref{lemma:ergodicity}, we require the following lemma.
\begin{lemma}[Theorem 9.6.1 in \citet{bakry2013analysis}]\label{lemma:bound_KL_w2_sobo}
Suppose the target density $\pi$ satisfies logarithmic Sobolev inequality with a positive constant $\lambda$, for any density $\mu$ it holds that
\begin{align*}
\cW_2(\mu, \pi)\le \sqrt{2\lambda D_{KL}(\mu||\pi)}.
\end{align*}
\end{lemma}

\begin{proof}[Proof of Lemma \ref{lemma:ergodicity}]
Recall that for LS-LD \eqref{eq:langevin_cont},  the distribution of $\bX_t$, denoted by $p(\xb, t)$, can be described by the following Fokker-Planck equation
\begin{align}\label{eq:fokker_planck}
\frac{\partial p(\xb, t)}{\partial t} &= \frac{1}{\beta}\la\Ab_\sigma^{-1}, \nabla^2 p(\xb,t)\ra + \la\nabla, p(\xb, t)\Ab_\sigma^{-1} \nabla f(\xb)\ra\notag\\
& = \frac{1}{\beta}\la\nabla, \Ab_\sigma^{-1}\nabla p(\xb,t)\ra + \la\nabla, p(\xb, t)\Ab_\sigma^{-1} \nabla f(\xb)\ra,
\end{align}
where $\la\nabla, \mathbf{h}(\xb)\ra$ denotes the divergence of the vector field $\mathbf{h}(\xb)$.
Let $\PP_t$ be the short-hand notation of $p(\xb, t)$, and denote by $D_{KL}(\PP_t||\pi)$ the KL-divergence between the distribution $\PP_t$ and the target distribution $\pi$. Then, we have
\begin{align*}
\frac{\dd D_{KL}(\PP_t||\pi)}{\dd t} &= \int_{\RR^d} \frac{\partial }{\partial t}\bigg[\PP_t\log\bigg(\frac{\PP_t}{\pi}\bigg)\bigg] \dd \xb \notag\\
&= \int_{\RR^d} \frac{\partial \PP_t}{\partial t}\big[\log(\PP_t) + 1 - \log(\pi)\big] \dd \xb.
\end{align*}
Similar to the proof of Proposition 2 in \citet{mou2017generalization}, by \eqref{eq:fokker_planck} we further have
\begin{align*}
\frac{\dd D_{KL}(\PP_t||\pi)}{\dd t}& = -\int_{\RR^d} \bigg\la\frac{1}{\beta}\Ab_\sigma^{-1}\nabla \PP_t + \PP_t\Ab_\sigma^{-1}\nabla f(\xb),  \nabla \log(\PP_t) - \nabla \log(\pi)\bigg\ra \dd\xb\notag\\
& = -\int_{\RR^d} \bigg\la\Ab_\sigma^{-1}\bigg(\frac{1}{\beta}\PP_t\nabla \log(\PP_t) +  \PP_t\nabla f(\xb)\bigg),  \nabla \log(\PP_t) - \nabla \log(\pi)\bigg\ra \dd\xb,
\end{align*}
where the second equality holds due to $\nabla \PP_t = \PP_t\nabla \log(\PP_t)$.
In addition, note that $\pi\propto e^{-\beta f(\xb)}$, we have $\nabla \log(\pi) = -\beta\nabla f(\xb) $. Then we have
\begin{align*}
\frac{\dd D_{KL}(\PP_t||\pi)}{\dd t}& = -\frac{1}{\beta}\int_{\RR_d}\big\la\Ab_\sigma^{-1}\big(\nabla \log(\PP_t) - \nabla \log(\pi)\big), \nabla \log(\PP_t) - \nabla \log(\pi)\big\ra \PP_t \dd\xb\notag\\
& = -\frac{1}{\beta}\int_{\RR_d}\|\nabla \log(\PP_t) - \nabla \log(\pi)\|_{\Ab_\sigma^{-1}}^2 \PP_t\dd\xb.
\end{align*}
Since $\Ab_\sigma$ is a positive definite matrix, there exists a constant $c_0\in[ \|\Ab_\sigma\|_2^{-1}, 1]$ such that 
\begin{align}\label{eq:bound_KL}
\frac{\dd D_{KL}(\PP_t||\pi)}{\dd t}\le -\frac{c_0}{\beta}\int_{\RR_d} \|\nabla \log(\PP_t) - \nabla \log(\pi) \|_2^2 \PP_t \dd \xb= - \frac{c_0}{\beta}\Ib(\PP_t||\pi),
\end{align}
where $\Ib(\PP_t||\pi)$ denotes the fisher information between $\PP_t$ and $\pi$. By Proposition \ref{prop:sobo}, we know that the target density $\pi$ satisfies logarithmic Sobolev inequality with constant $\lambda>0$. Then, from \citet{markowich1999trend}, we have
\begin{align*}
D_{KL}(\PP_t||\pi)\le \frac{1}{\lambda}  \Ib(\PP_t||\pi).
\end{align*}
Plugging the above inequality into \eqref{eq:bound_KL}, we obtain
\begin{align*}
\frac{\dd D_{KL}(\PP_t||\pi)}{\dd t}\le - \frac{c_0}{\lambda \beta}D_{KL}(\PP_t||\pi),
\end{align*}
which implies that 
\begin{align*}
 D_{KL}(\PP_t||\pi)\le  D_{KL}(\PP_0||\pi) e^{-c_0t/(\beta \lambda)}.
\end{align*}
Note that we have $\PP_0 =\delta(0)$, where $\delta(\cdot)$ is  the Dirac delta function, thus,
\begin{align*}
D_{KL}(\PP_0||\pi) = \int_{\RR^d} \PP_0\big[\log(\PP_0)-\log(\pi)\big]\dd \xb  = -\log(\pi)|_{\xb = 0} = \beta f(0)+\log(\Lambda),
\end{align*}
where $\Lambda = \int_{\RR^d} e^{-\beta f(\xb)} \dd\xb$.
Then by Lemma \ref{lemma:bound_KL_w2_sobo}, we have the following regarding the $2$-Wasserstein distance $\cW_2\big(\PP(\bX_{k\eta}), \pi\big)$,
\begin{align*}
\cW_2\big(\PP(\bX_{k\eta}), \pi\big)\le \sqrt{2\lambda D_{KL}(\PP(\bX_0)||\pi)}\cdot e^{-c_0t/(2\beta \lambda)} = \big[2\lambda\big(\beta f(0) + \log(\Lambda)\big)\big]^{1/2}\cdot e^{-c_0t/(2\beta \lambda)},
\end{align*}
which completes the proof.
\end{proof}

\subsection{Proof of Lemma \ref{lemma:discrete_error_convex}}

We first deliver the following useful lemmas.

\begin{lemma}\label{lemma:contraction}
Consider any two LS-LD sequences $\{\bW_t\}_{t\ge 0}$ and $\{\bV_t\}_{t\ge 0}$, and assume that $\bW_t$ and $\bV_t$ have shared Brownian motion terms. Under Assumption \ref{assump:convex}, for any $t>0$ it holds that,
\begin{align*}
\EE[\|\bW_t - \bV_t\|_{\Ab_\sigma}^2]\le \EE[\|\bW_0 - \bV_0\|_{\Ab_\sigma}^2].
\end{align*}
\end{lemma}

\begin{lemma}\label{lemma:discretization_error_lmc} 
Under Assumptions \ref{assump:dissipative} and \ref{assump:smooth}, if set the step size $\eta \le C m\beta^{-1}/M^2$ for some sufficiently small constant $C$, there exists a constant $\gamma_2 = d^{-1}\sum_{i=1}^d(1+2\sigma - 2\sigma \cos(2\pi i/d))^{-1}$ such that for any $\xb_k$ with $k\ge 0$,
\begin{align*}
\EE[\|\cL_\eta \xb_k - \cG_\eta \xb_k\|_{\Ab_\sigma}^2]\le4\gamma_2\beta^{-1} d \eta^3.
\end{align*}
\end{lemma}

\noindent Now we are ready to complete the proof of Lemma \ref{lemma:discrete_error_convex}.
\begin{proof}[Proof of Lemma \ref{lemma:discrete_error_convex}]

For the sake of simplicity, we first define three operators $\cL_t$, $\cG_t$ and $\cS_t$ as follows: for any $\xb\in \RR^d$ we denote by $\cL_t \xb$ the random point generated by LS-LD at time $t$ starting from $\xb$, $\cG_t \xb$ the point after performing one-step LS-SGLD with full gradient at $\xb$ with step size $t$, and $\cS_t \xb$ the point after performing one-step LS-SGLD at $\xb$ with step size $t$. Then we have
\begin{align}\label{eq:decomposition_discree}
\EE[\|\xb_K - \bX_{K\eta}\|_{\Ab_\sigma}^2] &= \EE[\|\xb_K - \cG_\eta \xb_{K-1} + \cG_\eta \xb_{K-1} - \bX_{K\eta}\|_{\Ab_\sigma}^2]\notag\\
& = \EE[\|\xb_K - \cG_\eta \xb_{K-1}\|_{\Ab_\sigma}^2] + \EE[\|\cG_\eta \xb_{K-1} - \bX_{K\eta}\|_{\Ab_\sigma}^2],
\end{align}
where the second equality follows from the fact that $\EE[\la\xb_K - \cG_\eta \xb_{K-1}, \Ab_\sigma(\cG_\eta \xb_{K-1} - \bX_{K\eta})\ra] = 0$ since at any iteration the randomness of stochastic gradient is independent of the iterate. Regarding the first term on the R.H.S. of \eqref{eq:decomposition_discree}, we have
\begin{align}\label{eq:bound_discrete_term1}
\EE[\|\xb_K - \cG_\eta \xb_{K-1}\|_{\Ab_\sigma}^2] 
& = \eta^2\EE[\|\cS_\eta \xb_{K-1} - \cG_\eta \xb_{K-1}\|_{\Ab_\sigma}^2]\notag\\
&\le   \eta^2\EE[\|\Ab_\sigma^{-1}\gb_{K-1} - \Ab_\sigma^{-1}\nabla f(\xb_{K-1})\|_{\Ab_\sigma}^2]\notag\\
&\le \frac{\eta^2}{B}\EE[\|\Ab_\sigma^{-1}\nabla f_i(\xb_{K-1}) - \Ab_\sigma^{-1}\nabla f(\xb_{K-1})\|_{\Ab_\sigma}^2]\notag\\
&\le \frac{\gamma_1 \eta^2 d\omega^2}{B},
\end{align}
where $\gamma_1 \in [\|\Ab_\sigma\|_2^{-1}, 1)$ is a problem-dependent parameter, the first inequality follows the definitions of operators $\cS_\eta$ and $\cG_\eta$, the second inequality follows from Lemma A.1 in \citet{lei2017nonconvex} and the last inequality is by Assumption \ref{assump:bound_var}.
In terms of the second term on the R.H.S. of \eqref{eq:decomposition_discree}, we have
\begin{align}\label{eq:bound_discrete_term2}
\EE[\|\cG_\eta \xb_{K-1} - \bX_{K\eta}\|_{\Ab_\sigma}^2] &= \EE[\|\cG_\eta \xb_{K-1} - \cL_\eta \xb_{K-1} + \cL_\eta \xb_{K-1} - \bX_{K\eta}\|_{\Ab_\sigma}^2]\notag\\
&\le (1+\alpha)\EE[\|\cG_\eta \xb_{K-1} - \cL_\eta \xb_{K-1}\|_{\Ab_\sigma}^2] + (1+1/\alpha) \EE[\|\cL_\eta \xb_{K-1} - \cL_\eta \bX_{(K-1)\eta}\|_{\Ab_\sigma}^2]\notag\\
&\le 4(1+\alpha)\gamma_2\beta^{-1} d\eta^3 + (1+1/\alpha) \EE[\| \xb_{K-1} - \bX_{(K-1)\eta}\|_{\Ab_\sigma}^2],
\end{align}
where $\alpha$ is a positive constant that will be specified later, the first inequality is by Young's inequality, and the second inequality follows from Lemmas \ref{lemma:contraction} and \ref{lemma:discretization_error_lmc}. Plugging \eqref{eq:bound_discrete_term1} and \eqref{eq:bound_discrete_term2} into \eqref{eq:decomposition_discree} gives
\begin{align*}
\EE[\|\xb_{K} - \bX_{K\eta}\|_{\Ab_\sigma}^2] &\le 4(1+\alpha)\gamma_2\beta^{-1} d\eta^3 + \frac{\gamma_1 \eta^2 d\omega^2}{B} + (1+1/\alpha) \EE[\|\xb_{K-1} - \bX_{(K-1)\eta}\|_{\Ab_\sigma}^2] .
\end{align*}
Then, by recursively applying the above inequality, we obtain
\begin{align*}
\EE[\|\xb_K - \bX_{K\eta}\|_{\Ab_\sigma}^2] &\le (1+1/\alpha)^K \EE[\|\xb_{0} - \bX_0\|_{\Ab_\sigma}^2] +  \sum_{k=0}^{K-1} (1+1/\alpha)^k \bigg[4(1+\alpha)\gamma_2\beta^{-1} d\eta^3 + \frac{\gamma_1 \eta^2 d\omega^2}{B}\bigg]\notag\\
&= \alpha \big[(1+1/\alpha)^K - 1\big] \cdot \bigg[4(1+\alpha)\gamma_2\beta^{-1} d\eta^3 + \frac{\gamma_1 \eta^2 d\omega^2}{B}\bigg].
\end{align*}
Let $\alpha = K$ and apply the inequality $(1+1/K)^K-1\le e-1\le 2$, the above inequality implies
\begin{align*}
\EE[\|\xb_K - \bX_{K\eta}\|_{\Ab_\sigma}^2]\le 2 K\eta^2\cdot\bigg[\frac{\gamma_1d\omega^2}{B} + 4(K+1)\gamma_2 \beta^{-1} d\eta \bigg].
\end{align*}
Based on the definition of $2$-Wasserstein distance, we have
\begin{align*}
\cW_2^2\big(\PP(\xb_K) , \PP(\bX_{K\eta})\big) \le  \sqrt{\EE[\|\xb_K - \bX_{K\eta}\|_2^2]}\le\sqrt{\EE[\|\xb_K - \bX_{K\eta}\|_{\Ab_\sigma}^2]}\le \bigg(\frac{2\gamma_1d\omega^2K\eta^2}{B}\bigg)^{1/2} +\big[8\gamma_2K(K+1) \beta^{-1} d\eta ^3\big]^{1/2},
\end{align*}
where the last inequality is by the fact that $\sqrt{x^2+y^2}\le |x|+|y|$. This completes the proof.


\end{proof}

\subsection{Proof of Lemma \ref{lemma:discretization}}

In order to prove Lemma \ref{lemma:discretization}, we require the following lemmas.

\begin{lemma}\label{lemma:bound_norm}
Under Assumptions \ref{assump:dissipative} and \ref{assump:smooth}, for all $k\ge 0$, there exists a constant $c_1\in[\|\Ab_\sigma\|_2^{-1}, 1)$ such that
\begin{align*}
\EE[\|\xb_k\|_2^2]\le\EE[\|\xb_k\|_{\Ab_\sigma}^2]\le \frac{2(2b + \beta^{-1} d)}{c_1 m}.
\end{align*}
\end{lemma}

\begin{lemma}[Theorem 2.3 in \citet{bolley2005weighted}]\label{lemma:bound_KL_w2}
Let $\mu,\nu$ be two probability measures with finite exponential second moments, it holds that
\begin{align*}
\cW_2(\mu,\nu)\le \Gamma \Big[\sqrt{D_{KL}(\mu||\nu)} + \big[D_{KL}(\mu||\nu)\big]^{1/4}\Big],
\end{align*}
where
\begin{align*}
\Gamma = \inf_{\alpha>0}\bigg(\frac{1}{\alpha}\Big(\frac{3}{2} + \log \EE_\nu[e^{\alpha\|\xb\|_2^2}]\Big)\bigg)^{1/2}.
\end{align*}
\end{lemma}

\begin{lemma}[Lemma 4 in \citet{Wang:2019:DPLSSGD}]\label{lemma:gaussian_norm}
Let $\bepsilon \sim \cN(0,\Ib)$ be the standard Gaussian random vector with dimension $d$, it holds that
\begin{align*}
\EE\big[\|\Ab_\sigma^{-1}\bepsilon\|_2^2\big] = \sum_{i=1}^d \frac{1}{\big(1 + 2\sigma - 2\sigma \cos(2\pi i/d)\big)^2}.
\end{align*}
\end{lemma}

\begin{lemma}\label{lemma:bound_exponential_moment}
Under Assumptions \ref{assump:dissipative} and \ref{assump:smooth}, let $\bX_t$ denote the solution of LS-LD \eqref{eq:langevin_cont} at time $t$ with initial point $\bX_0 = 0$. Then if the inverse temperature satisfies $\beta\ge 2\|\Ab_\sigma\|_2/m$, it holds that
\begin{align*}
\EE[e^{\|\bX_t\|_2^2}]\le e^{2(b+\beta^{-1} d)t}.
\end{align*}
\end{lemma}

Based on the above lemmas, we are able to complete the proof of Lemma \ref{lemma:discretization}.
\begin{proof}[Proof of Lemma \ref{lemma:discretization}]
By Lemma \ref{lemma:bound_KL_w2}, we know that the $2$-Wasserstein distance between any two probability measures can be bounded by their KL divergence. Therefore, the remaining part will focus on deriving the upper bound of the KL divergence $D_{KL}\big(\PP(\xb_k)||\PP(\bX_{k\eta})\big)$. Similar to the proof technique used in \citet{dalalyan2017theoretical,raginsky2017non,xu2018global}, we leverage the following continuous-time interpolation of LS-SGLD
\begin{align}\label{eq:interpolation1}
\tilde \bX_t = \int_{0}^t-\Ab_\sigma^{-1}\Gb_s \dd s+ \int_{0}^t\sqrt{2\beta^{-1}}\Ab_\sigma^{-1/2}\dd\bB_s,
\end{align}
where $\Gb_t = \sum_{k=0}^\infty \gb_k \ind \{t\in[k\eta, (k+1)\eta)\}$. It can be easily verified that $\tilde \bX_{k\eta}$ follows the same distribution as $\xb_k $. However, it is worth noting that \eqref{eq:interpolation1} does not form a Markov chain since it contains randomness of the stochastic gradient. To tackle this, we leverage the results in \citet{gyongy1986mimicking} and construct the following Markov chain to mimic \eqref{eq:interpolation1},
\begin{align*}
\hat \bX_t = \int_{0}^t-\Ab_\sigma^{-1}\hat\Gb_s \dd s+ \int_{0}^t\sqrt{2\beta^{-1}}\Ab_\sigma^{-1/2}\dd\bB_s,
\end{align*}
where $\hat \Gb_s = \EE[\Gb_s|\hat \bX_s = \tilde \bX_s ]$. It was shown that $\hat \bX_t$ and $\tilde \bX_t$ has the same one-time marginal distribution \citep{gyongy1986mimicking}. Then let $\PP_t$ and $\QQ_t$ denote the distribution of $\bX_t$ and $\hat \bX_t$ respectively, by Girsanov formula, the Radon-Nikodym derivative of $\PP_t$ with respect to $\QQ_t$ can be derived as follows,
\begin{align*}
\frac{\dd \PP_t}{\dd \QQ_t} = \exp\bigg\{\frac{\beta}{2}\int_{0}^t \la\nabla f(\hat\bX_s) -\hat \Gb_s, \Ab_\sigma^{-1/2}\dd \bB_s \ra - \frac{\beta}{4} \int_{0}^t \|\Ab_\sigma^{-1}\nabla f(\hat\bX_s) - \Ab_\sigma^{-1}\hat \Gb_s\|_2^2 \dd s\bigg\}.
\end{align*}
Therefore, let $T = K\eta$, the KL divergence $D_{KL}(\PP_T||\QQ_T)$ satisfies
\begin{align*}
D_{KL}(\QQ_T||\PP_T) &= -\int_{\RR^d} \log\bigg(\frac{ \dd\PP_T}{\dd\QQ_T}\bigg) \dd \QQ_T  \notag\\
& = \frac{\beta}{4} \int_{0}^T \EE\big[\|\Ab_\sigma^{-1}\nabla f(\hat\bX_s) - \Ab_\sigma^{-1}\hat \Gb_s\|_2^2 \big] \dd s\notag\\
& = \frac{\beta}{4} \sum_{k=0}^{K-1}\int_{k\eta}^{(k+1)\eta} \EE\big[\|\Ab_\sigma^{-1}\nabla f(\tilde\bX_s) - \Ab_\sigma^{-1}\gb_k\|_2^2 \big] \dd s,
\end{align*}
where the second equality holds due to $\EE[\la\nabla f(\hat\bX_s)-\hat \Gb_s, \Ab_\sigma^{-1/2}\dd \bB_s\ra] = 0$ and the second equality follows from the fact that $\hat\bX_s$ and $\tilde \bX_s$ follow the same distribution. Using Young's inequality, we have
\begin{align*}
\EE[\|\Ab_\sigma^{-1}\nabla f(\tilde \bX_s) - \Ab_\sigma^{-1}\gb_k\|_2^2]&\le 2 \underbrace{\EE[\|\Ab_\sigma^{-1}\nabla f(\tilde \bX_s) - \Ab_\sigma^{-1}\nabla f(\xb_k)\|_2^2]}_{I_1} + 2 \underbrace{\EE[\|\Ab_\sigma^{-1}\nabla f(\xb_k) - \Ab_\sigma^{-1}\gb_k\|_2^2]}_{I_2}.
\end{align*}
Then we are going to tackle $I_1$ and $I_2$ separately. Note that $\|\Ab_\sigma^{-1}\|_2\le 1$, thus by Assumption \ref{assump:smooth}, we have the following for $I_1$,
\begin{align*}
I_1&\le \EE[\|\nabla f(\tilde \bX_s) - \nabla f(\xb_k)\|_2^2]\le M^2\EE[\|\tilde \bX_s - \xb_k\|_2^2].
\end{align*}
Based on the definition of $\tilde \bX_s$, we have $\tilde \bX_s - \xb_k = (s-k\eta)\Ab_\sigma^{-1}\gb_k +  \sqrt{2\beta^{-1}(s-k \eta)}\Ab_\sigma^{-1/2}\bepsilon_k$. Since $s-k\eta\le \eta$, it follows that
\begin{align*}
I_1\le M^2\EE[\|\tilde \bX_s - \xb_k\|_2^2]\le \eta^2M^2 \EE[\|\Ab_\sigma^{-1}\gb_k\|_2^2] + 2\eta M^2\beta^{-1} \EE[\|\Ab_\sigma^{-1/2}\bepsilon_k\|_2^2].
\end{align*}
Regarding $I_2$, based on Lemma A.1 in \citet{lei2017nonconvex} and Assumption \ref{assump:bound_var}, we have
\begin{align*}
I_2 \le \frac{1}{B} \EE[\|\Ab_\sigma^{-1}\nabla f(\xb_k) - \Ab_\sigma^{-1}\nabla f_i(\xb_k)\|_2^2]\le \frac{\gamma_1 d\omega^2}{B},
\end{align*}
where $\gamma_1\in[\|\Ab_\sigma\|_2^{-2}, 1)$ is a problem-dependent parameter.
Putting everything together, we have
\begin{align*}
D_{KL}(\QQ_T||\PP_T)\le\sum_{k=0}^{K-1} \eta \bigg\{\frac{\beta}{2}\eta^2M^2\EE\big[\|\gb_k\|_{\Ab_\sigma^{-2}}^2\big] + \eta M^2\EE\big[\|\bepsilon_k\|_{\Ab_\sigma^{-1}}^2\big] + \frac{\gamma_1\beta d\omega^2}{2B}\bigg\}.
\end{align*}
By Lemma \ref{lemma:bound_grad} and Young's inequality, we know that 
\begin{align*}
\EE\big[\|\gb_k\|_{\Ab_\sigma^{-2}}^{2}\big]\le \EE[\|\gb_k\|_2^2] \le 2M^2\EE[\|\xb_k\|_2^2] + 2G^2\le \frac{4M^2(2b + \beta^{-1}d)}{c_1m} + 2G^2,
\end{align*}
where $G = \max_{i\in[n]} \|\nabla f_i(0)\|_2$. Then by Lemma \ref{lemma:gaussian_norm}, we have $\EE[\|\bepsilon_k\|_{\Ab_\sigma^{-1}}^2]\le \gamma_2 d$ with $\gamma_2 = d^{-1}\sum_{i=1}^d(1+2\sigma - 2\sigma \cos(2\pi i/d))^{-1}$,  which is strictly smaller than $1$. Therefore,
\begin{align*}
D_{KL}(\QQ_T||\PP_T)\le \frac{2\beta M^4(2b + \beta^{-1}d)+\beta c_1 m M^2 G^2}{c_1m}K\eta^3 + \gamma_2  M^2 dK\eta^2 + \frac{\gamma_1\beta d\omega^2}{2B}K\eta.
\end{align*}
For sufficiently small step size such  that 
\begin{align*}
\eta \le \frac{c_1\beta^{-1}\gamma_2  m d }{2 M^2(2b + \beta^{-1}d)+ c_1 m  G^2},
\end{align*}
we have
\begin{align*}
D_{KL}(\PP(\xb_K)||\PP(\bX_{K\eta}))\le  \frac{\gamma_1\beta d\omega^2}{2B}K\eta+2\gamma_2 M^2 dK\eta^2.
\end{align*}
Then, by Lemma \ref{lemma:bound_KL_w2}, we have
\begin{align*}
\cW_2\big(\PP(\xb_K),\PP(\bX_{K\eta})\big)\le\Gamma \Big[\sqrt{D_{KL}\big(\PP(\xb_K)||\PP(\bX_{K\eta})\big)} + \big[D_{KL}\big(\PP(\xb_K)||\PP(\bX_{K\eta})\big)\big]^{1/4}\Big],
\end{align*}
where $\Gamma$ can be further bounded as
\begin{align*}
\Gamma \le \Big(\frac{3}{2} + \log \EE[e^{\|\bX_{K\eta}\|_2^2}]\Big)^{1/2}\le \big(3/2 + 2(b + \beta d)K\eta\big)^{1/2}\le \big(3/2 + 2(b + \beta d)\big)^{1/2}\cdot(K\eta)^{1/2},
\end{align*}
where the first inequality is by the choice $\alpha = 1$, the second inequality is by Lemma \ref{lemma:bound_exponential_moment} and the last inequality is by the assumption that $K\eta>1$. Therefore, define by $\bar \Gamma=\big(3/2 + 2(b + \beta^{-1} d)\big)^{1/2}$ , the $2$-Wasserstein distance $\cW_2\big(\PP(\xb_K),\PP(\bX_{K\eta})\big)$ can be bounded by
\begin{align*}
\cW_2\big(\PP(\xb_K),\PP(\bX_{K\eta})\big)\le\bar \Gamma (K\eta)^{1/2} \bigg[\bigg(\frac{\gamma_1\beta d\omega^2}{2B}K\eta+2\gamma_2 M^2 dK\eta^2\bigg)^{1/2} + \bigg(\frac{\gamma_1\beta d\omega^2}{2B}K\eta+2\gamma_2 M^2 dK\eta^2\bigg)^{1/4}\bigg],
\end{align*}
which completes the proof.
\end{proof}


\section{Proof of Lemmas in Appendix \ref{sec:proof_lemma}}

\subsection{Proof of Lemma \ref{lemma:contraction}}
\begin{proof}[Proof of Lemma \ref{lemma:contraction}]
Assuming shared Brownian motions in $\bW_t$ and $\bV_t$, we have
\begin{align*}
\dd \EE[\|\bW_t - \bV_t\|_{\Ab_\sigma}^2] &= -\EE\big[\big\la\Ab_\sigma^{-1}\big(\nabla f(\bW_t) - \nabla f(\bV_t)\big),  \Ab_\sigma(\bW_t - \bV_t)\big\ra\big]\dd t\notag\\
& = -\EE\big[\big\la\big(\nabla f(\bW_t) - \nabla f(\bV_t)\big),  \bW_t - \bV_t\big\ra\big]\dd t \notag\\
& \ge 0,
\end{align*}
where the first equality follows from the fact that we assume shared Brownian motion terms on both dynamics $\{\bW_t\}_{t\ge 0}$ and $\{\bV_t\}_{t\ge 0}$ and
the inequality is due to the convexity of $f(\xb)$.
Therefore, it can be evidently concluded that
\begin{align*}
\EE[\|\bW_t - \bV_t\|_{\Ab_\sigma}^2]\le \EE[\|\bW_0 - \bV_0\|_{\Ab_\sigma}^2],
\end{align*}
which completes the proof.
\end{proof}

\subsection{Proof of Lemma \ref{lemma:discretization_error_lmc}}
\begin{proof}[Proof of Lemma \ref{lemma:discretization_error_lmc}]
To simplify the analysis, let $\xb$ be any iterate of LS-SGLD and define $\xb = \bX_0$. Then the operators $\cG_\eta$ and $\cL_\eta$ satisfy
\begin{align*}
\cG_\eta\xb &= \bX_0 - \eta\Ab_\sigma^{-1}\nabla f(\bX_0) + \sqrt{2\beta^{-1}\eta}\Ab_\sigma^{-1/2}\bepsilon\notag\\
& = \bX_0 -\int_{0}^\eta \Ab_\sigma^{-1}\nabla f(\bX_0) \dd t + \int_{0}^\eta \sqrt{2\beta^{-1}}\Ab_\sigma^{-1/2} \dd \bB_t;\notag\\
\cL_\eta\xb 
& = \bX_0 -\int_{0}^\eta \Ab_\sigma^{-1}\nabla f(\bX_t) \dd t + \int_{0}^\eta \sqrt{2\beta^{-1}}\Ab_\sigma^{-1/2} \dd \bB_t.
\end{align*}
Consider synchronous Brownian terms in $\cG_\eta$ and $\cL_\eta$, we have
\begin{align}\label{eq:decomposition_onestep_LMC}
\EE[\|\cL_\eta \xb - \cG_\eta \xb\|_{\Ab_\sigma}^2] &= \EE\bigg[\bigg\|\int_{0}^\eta\big[\Ab_\sigma^{-1}\nabla f(\bX_0) -  \Ab_\sigma^{-1}\nabla f(\bX_t)\big]\dd t\bigg\|_{\Ab_\sigma}^2\bigg]\notag\\
& \le \EE\bigg[\eta \int_{0}^\eta \big\|\Ab_\sigma^{-1}\big[\nabla f(\bX_0) - \nabla f(\bX_t)\big]\big\|_{\Ab_\sigma}^2\dd t\bigg]\notag\\
&\le M^2\bigg[\eta\int_{0}^\eta\EE[\|\bX_t - \bX_0\|_2^2]\dd t\bigg],
\end{align}
where the second inequality follows from Jensen's inequality and the last inequality follows from Assumption \ref{assump:smooth} and the fact that $\|\Ab_\sigma\|_2\ge1$. We further have
\begin{align*}
\EE[\|\bX_t - \bX_0\|_2^2] &= \EE\bigg[\bigg\|\int_{0}^t \Ab_\sigma^{-1}\nabla f(\bX_\tau) \dd \tau \bigg\|_2^2\bigg] + 2\beta^{-1} t\EE[\|\Ab_\sigma^{-1/2}\bepsilon\|_2^2]\notag\\
&\le \EE\bigg[t\int_{0}^t \|\nabla f(\bX_\tau)\|_2^2 \dd \tau\bigg] + 2\gamma_2\beta^{-1} dt
\end{align*}
where the inequality is by Jensen's inequality and Lemma \ref{lemma:gaussian_norm} and $\gamma_2 = d^{-1}\sum_{i=1}^d(1+2\sigma - 2\sigma \cos(2\pi i/d))^{-1}$ is strictly smaller than $1$. By Lemma \ref{lemma:bound_grad}, we have
\begin{align*}
\EE[\|\nabla f(\bX_\tau)\|_2^2] \le 2M^2\EE[\|\bX_\tau\|_2^2] + 2G^2.
\end{align*}
Note that by Ito's lemma we have for any $0\le s\le \tau$,
\begin{align*}
\frac{\dd \EE[\|\bX_s\|_{\Ab_\sigma}^2]}{\dd s} = -2\EE[\la\bX_s,\nabla f(\bX_s)\ra] + \beta^{-1} d\le -2m\EE[\|\bX_s\|_2^2]+2b+\beta^{-1}d\le 2b+\beta^{-1}d,
\end{align*}
where the second inequality follows from Assumption \ref{assump:dissipative}. Therefore, 
\begin{align*}
\EE[\|\bX_\tau\|_2^2] \le \EE[\|\bX_\tau\|_{\Ab_\sigma}^2] = \EE[\|\bX_0\|_{\Ab_\sigma}^2]+\int_{0}^\tau\frac{\dd\EE[\|\bX_s\|_{\Ab_\sigma}^2]}{\dd s}\dd s\le \EE[\|\bX_0\|_{\Ab_\sigma}^2] + \tau(2b+\beta^{-1 }d).
\end{align*}
Note that $\bX_0 = \xb$ is a iterate of LS-SGLD, by Lemma \ref{lemma:bound_norm} we have $\EE[\|\bX_0\|_{\Ab_\sigma}^2]\le (2b+\beta^{-1}d)/(c_1m)$ for some constant $c_1\in[\|\Ab_\sigma\|_2^{-1},1]$. Therefore,
\begin{align*}
\EE[\|\nabla f(\bX_\tau)\|_2^2] \le 2M^2\EE[\|\bX_\tau\|_2^2] + 2G^2\le \frac{4M^2(2b + \beta^{-1}d)}{c_1m} + 2G^2 + 2M^2\tau(2b+\beta^{-1 }d).
\end{align*}
Thus, it follows that
\begin{align*}
\EE[\|\bX_t - \bX_0\|_2^2]\le \bigg(\frac{4M^2(2b + \beta^{-1}d)}{c_1m} + 2G^2+2M^2\tau(2b+\beta^{-1 }d)\bigg)t^2 ++ 2\gamma_2\beta^{-1}  dt.
\end{align*}
Note that $\tau,t\le \eta$, plugging the above inequality into \eqref{eq:decomposition_onestep_LMC}, we have
\begin{align*}
\EE[\|\cL_\eta \xb - \cG_\eta \xb\|_2^2] \le M^2\bigg[\bigg(\frac{4M^2(2b + \beta^{-1}d)}{c_1m} + 2G^2+2M^2(2b+\beta^{-1 }d)\eta\bigg)\eta^4 + 2\gamma_2\beta^{-1} d \eta^3\bigg].
\end{align*}
For sufficiently small step size satisfying
\begin{align*}
\eta \le \frac{c_1\beta^{-1}\gamma_2  m d }{4 M^2(2b + \beta^{-1}d)+ 2c_1 m  G^2}\wedge \sqrt{\frac{\gamma_2\beta^{-1}d}{M^2(2b+\beta^{-1}d)}}, 
\end{align*}
we have 
\begin{align*}
\EE[\|\cL_\eta \xb - \cG_\eta \xb\|_2^2] \le 4\gamma_2\beta^{-1} d \eta^3.
\end{align*}
This completes the proof.
\end{proof}

\subsection{Proof of Lemma \ref{lemma:bound_norm}}

\begin{lemma}[Lemma 3.1 in \citet{raginsky2017non}]\label{lemma:bound_grad}
For any $\xb\in\RR^d$ and $i\in[n]$, it holds that
\begin{align*}
\|\nabla f_i(\xb)\|_2\le M\|\xb\|_2 + G,
\end{align*}
where $G = \max_{i\in[n]} \|\nabla f_i(0)\|_2$.
\end{lemma}

\begin{proof}
Recall the update formula of $\xb_k$,
\begin{align*}
\xb_{k+1} = \xb_k - \eta \Ab_\sigma^{-1} \gb_k + 
\sqrt{2\beta^{-1}\eta}\Ab_\sigma^{-1/2}\bepsilon_k.
\end{align*}
Therefore, it holds that
\begin{align*}
\EE[\|\xb_{k+1}\|_{\Ab_\sigma}^2] = &\EE[\|\xb_k - \eta \Ab_\sigma^{-1} \gb_k\|_{\Ab_\sigma}^2] + 2\eta\beta^{-1}\EE[\|\Ab_\sigma^{-1/2}\bepsilon_k\|_{\Ab_{\sigma}}^2] \notag  \\
& = \EE[\|\xb_k\|_{\Ab_\sigma}^2] - 2\eta\EE[\la\xb_k, \gb_k\ra] + \eta^2 \EE\big[\|\gb_k\|_{\Ab_\sigma^{-1}}^2\big]+2\eta \beta^{-1} d,
\end{align*}
where the second equality follows from the fact that $\EE[\|\bepsilon_k\|_2^2] = d$. Note that all eigenvalues of $\Ab_\sigma$ are greater than $1$, it follows that
\begin{align*}
\EE[\|\xb_{k+1}\|_{\Ab_\sigma}^2] &= \EE[\|\xb_k\|_{\Ab_\sigma}^2] -2\eta\EE[\la\xb_k,\nabla f(\xb)\ra] + \eta^2\EE[\|\gb_k\|_2^2] + 2\eta\beta^{-1}d\notag\\
&\le \EE[\|\xb_k\|_{\Ab_\sigma}^2] - 2\eta m \EE[\|\xb_k\|_2^2] + 2 \eta b + 2\eta^2(M^2\EE[\|\xb_k\|_2^2] + G^2 ) + 2\eta\beta^{-1}d,
\end{align*}
where the inequality follows from Assumption \ref{assump:dissipative}, Lemma \ref{lemma:bound_grad} and Young's inequality. 
Since the step size  $\eta$ satisfies $\eta \le m/(2M^2)$, we further have
\begin{align*}
\EE[\|\xb_{k+1}\|_{\Ab_\sigma}^2]\le \EE[\|\xb_k\|_{\Ab_\sigma}^2] - \eta m \EE[\|\xb_k\|_2^2] + 2\eta(b + \beta^{-1} d + \eta G^2).
\end{align*}
Recall that all eigenvalues of $\Ab_\sigma$ are greater than $1$, there exists a constant $\|\Ab_\sigma\|_2^{-1}\le c_1\le 1$ such that 
\begin{align}\label{eq:contraction_norm}
\EE[\|\xb_{k+1}\|_{\Ab_\sigma}^2]\le (1-c_1\eta m)\EE[\|\xb_k\|_{\Ab_\sigma}^2] + 2\eta(b + \beta^{-1} d + \eta G^2).
\end{align}
Since $\eta\le 1/(c_1m)\wedge b/G$, \eqref{eq:contraction_norm} implies that the following holds for all $k\ge 0$,
\begin{align*}
\EE[\|\xb_{k}\|_{\Ab_\sigma}^2]\le (1 - c_1\eta m )^{k} \|\xb_0\|_{\Ab_\sigma}^2 + \frac{2(2b+\beta^{-1 }d)}{c_1m}.
\end{align*}
Since at the initialization $\xb_0 = 0$, we have
\begin{align*}
\EE[\|\xb_k\|_2^2]\le \EE[\|\xb_k\|_{\Ab_\sigma}^2]\le \frac{2(2b + \beta^{-1} d)}{c_1 m}.
\end{align*}
This completes the proof.
\end{proof}

\subsection{Proof of Lemma \ref{lemma:bound_exponential_moment}}
\begin{proof}
We first define the function $L(t) = e^{\|\bX_t\|_{\Ab_\sigma}^2}$, then by Ito's formula, we have
\begin{align*}
\dd \EE[L(t)] &= -2\EE[\la\Ab_\sigma\bX_t, \Ab_\sigma^{-1}\nabla f(\bX_t)\ra L(t)] \dd t + \EE[\la 4\Ab_\sigma \bX_t\bX_t^\top\Ab_\sigma + 2\Ab_\sigma, \beta^{-1}\Ab_\sigma^{-1}\Ib\ra L(t)]\dd t \notag\\
& = -2\EE\big[\big(\la\bX_t, \nabla f(\bX_t)\ra - \beta^{-1} d - 2\beta^{-1} \|\bX_t\|_{\Ab_\sigma}^2\big)L(t)\big] \dd t.
\end{align*}
By Assumption \eqref{assump:dissipative}, we further have
\begin{align*}
\dd \EE[L(t)] \le 2 \EE\big[\big((-m\|\bX_t\|_2^2 + 2\beta^{-1} \|\bX_t\|_{\Ab_\sigma}^2)+ b + \beta^{-1} d\big) L(t)\big] \dd t.
\end{align*}
Therefore, assume $\beta\ge 2\|\Ab_\sigma\|_2/m$, we have
\begin{align*}
\dd \EE[L(t)] \le 2(b+\beta^{-1} d) \EE[L(t)] \dd t.
\end{align*}
Since $L(t)$ is always positive, it holds that
\begin{align*}
\EE[L(t)]\le L(0)e^{2(b+\beta^{-1} d)t}.
\end{align*}
Note that $\|\bX_t\|_{\Ab_\sigma}^2\ge \|\bX_t\|_2^2$, we immediately have
\begin{align*}
\EE[e^{\|\bX_t\|_2^2}]\le \EE[e^{\|\bX_t\|_{\Ab_\sigma}^2}]\le L(0)e^{2(b+\beta^{-1} d)t}, 
\end{align*}
which completes the proof.
\end{proof}


\end{document}